\documentclass[runningheads,a4paper]{llncs}

\usepackage{amssymb}
\usepackage{graphicx}

\usepackage{graphicx} 
\usepackage{subfigure}

\usepackage{amsmath}
\usepackage{multirow}

\usepackage{latexsym}
\usepackage{mathrsfs}
\usepackage{amsmath}

\usepackage{enumerate}

\usepackage{url}
\usepackage{algorithm}
\usepackage{algorithmic}

\usepackage{shortbold}

\newcommand{\keywords}[1]{\par\addvspace\baselineskip
\noindent\keywordname\enspace\ignorespaces#1}

\begin{document}

\mainmatter  

\title{Adaptive Stochastic Primal-Dual Coordinate Descent for Separable Saddle Point Problems}

\titlerunning{Adaptive Stochastic Primal-Dual Coordinate Descent}

%
%
\author{Zhanxing Zhu%
\and Amos J. Storkey}

\authorrunning{Z. Zhu and A.J. Storkey}

\institute{Institute of Adaptive Neural Computation, School of Informatics, \\The University of Edinburgh, Edinburgh, EH8 9AB, UK \\
\email{\{zhanxing.zhu, a.storkey\}@ed.ac.uk}
}

%
%

\toctitle{Lecture Notes in Computer Science}
\tocauthor{Authors' Instructions}
\maketitle

\begin{abstract}
We consider a generic convex-concave saddle point problem with \emph{separable} structure, a form that covers a wide-ranged machine learning applications. Under this problem structure, we follow the framework of primal-dual updates for saddle point problems, and incorporate stochastic block coordinate descent with \emph{adaptive} stepsize into this framework. We theoretically show that our proposal of adaptive stepsize potentially achieves a sharper linear convergence rate compared with the existing methods. Additionally, since we can select ``mini-batch" of block coordinates to update, our method is also amenable to \emph{parallel} processing for large-scale data. We apply the proposed method to regularized empirical risk minimization and show that it performs comparably or, more often, better than state-of-the-art methods on both synthetic and real-world data sets.
\end{abstract}

\section{Introduction}
The generic convex-concave saddle point problem is written as
\begin{equation}
\min_{\xB \in \Rbb^d} \max_{\yB \in \Rbb^{q}} \left\{ L(\xB, \yB) = g(\xB) + \left<  \xB, \KB \yB \right> - \phi^{*}(\yB) \right\}, \label{eq:origin}
\end{equation}
where $g(\xB)$ is a proper convex function, $\phi^*(\cdot)$ is the convex conjugate of a convex function $\phi(\cdot)$, and matrix $\KB \in \Rbb^{d \times q}$. Many machine learning tasks reduce to solving a problem of this form \cite{jacob2009,hastie2009}. As a result, this saddle problem has been widely studied \cite{zhu2008,tseng2008,esser2010,chambolle2011,he2012,he2014}.

One important subclass of the general convex concave saddle point problem is where $g(\xB)$ or $\phi^*(\yB)$ exhibits an additive
{separable structure}. We say $\phi^*(\yB)$ is \emph{separable} when $\phi^*(\yB) = \frac{1}{n}\sum_{i=1}^n \phi_i^*(\yB_i)$, with $\yB_i \in \Rbb^{q_i}$, and $\sum_{i=1}^n q_i = q$. Separability for $g(\cdot)$ is defined likewise. To keep the consistent notation for the machine learning applications discussed later, we introduce matrix $\AB$ and let $\KB = \frac{1}{n} \AB$. Then we partition matrix $\AB$ into $n$ column blocks $\AB_i \in \Rbb^{d \times q_i}$, $i = 1,\dots,n$,  and $\KB \yB = \frac{1}{n}\sum_{i=1}^n  \AB_i \yB_i$, resulting in a problem of the form
\begin{equation}
\min_{\xB \in \Rbb^d} \max_{\yB \in \Rbb^{q}}  \left\{ L(\xB, \yB) =  g(\xB) +  \frac{1}{n}\sum_{i=1}^n \left( \left<  \xB, \AB_i \yB_i \right> - \phi_i^{*}(\yB_i) \right) \right\} \label{eq:sepccsp}
\end{equation}
for $\phi^*(\cdot)$ separable. We call any problem of the form (\ref{eq:origin}) where $g(\cdot)$ or $\phi^{*}(\cdot)$ has separable structure, a Separable Convex Concave Saddle Point (\emph{Sep-CCSP}) problem. Eq.~(\ref{eq:sepccsp}) gives the explicit form for when $\phi^*(\cdot)$ is separable.

In this work, we further assume that each $\phi_i^*(\yB_i)$ is $\gamma$-strongly convex, and $g(\xB)$ is $\lambda$-strongly convex, i.e.,
\begin{align*}
\phi_i^*(\yB_i ') & \geq \phi_i^*(\yB_i) + \nabla \phi^*(\yB_i) ^T \left( \yB_i ' - \yB_i \right) + \frac{\gamma}{2}\| \yB_i ' - \yB_i \|_2^2, \quad \forall \yB_i, \yB_i ' \in \Rbb^{q_i}\\
g(\xB ') & \geq g(\xB) + \nabla g(\xB) ^T \left( \xB ' - \xB \right) + \frac{\lambda}{2}\| \xB_i ' - \xB_i \|_2^2, \quad \forall \xB, \xB ' \in \Rbb^d,
\end{align*}
where we use $\nabla$ to denote both the gradient for smooth function and subgradient for non-smooth function. When the strong convexity cannot be satisfied, a small strongly convex perturbation can be added to make the problem satisfy the assumption \cite{zhang2014}.

One important instantiation of the Sep-CCSP problem in machine learning is the regularized empirical risk minimization (ERM, \cite{hastie2009}) of linear predictors,
\begin{equation}
\min_{\xB \in \Rbb^d} \left\{ J(\xB) = \frac{1}{n} \sum_{i= 1}^n \phi_i(\aB_i^T \xB) + g(\xB) \right\},
\end{equation}
where $\aB_1, \dots, \aB_n \in \Rbb^d$ are the feature vectors of $n$ data samples, $\phi_i(\cdot)$ corresponds the convex loss function w.r.t. the linear predictor $\aB_i^T \xB$, and $g(\xB)$ is a convex regularization term. Many practical classification and regression models fall into this regularized ERM formulation, such as linear support vector machine (SVM), regularized logistic regression and ridge regression, see \cite{hastie2009} for more details.

Reformulating the above regularized ERM by employing conjugate dual of the function $\phi_i(\cdot)$, i.e.
\begin{equation}
\phi_i^{*}(\aB_i^T \xB) = \max_{y_i \in \Rbb} \left< \xB, y_i \aB_i \right> - \phi_i^*(y_i), \label{eq:conjugate}
\end{equation}
leads directly to the following Sep-CCSP problem
\begin{equation}
\min_{\xB \in \Rbb^d} \max_{\yB \in \Rbb^{n}}   g(\xB) +  \frac{1}{n}\sum_{i=1}^n \left( \left<  \xB, y_i \aB_i \right> - \phi_i^{*}(y_i) \right). \label{eq:ermsaddle}
\end{equation}
Comparing with the general form, we note that the matrix $\AB_i$ in (\ref{eq:sepccsp}) is now a vector $\aB_i$. 
For solving the general saddle point problem (\ref{eq:origin}), many primal-dual algorithms can be applied, such as \cite{zhu2008,esser2010,chambolle2011,he2012,he2014}. In addition, the saddle point problem we consider can also be formulated as a composite function minimization problem and then solved by Alternating Direction Method of Multipliers (ADMM) methods \cite{ouyang2014}.

To handle the Sep-CCSP problem particularly for regularized ERM problem (\ref{eq:ermsaddle}), Zhang and Xiao \cite{zhang2014} proposed a  stochastic primal-dual coordinate descent (SPDC) method. SPDC applies stochastic coordinate descent method \cite{nesterov2012efficiency,richtarik2012parallel,richtarik2014iteration} into the primal-dual framework, where in each iteration a random subset of dual coordinates are updated. This method inherits the efficiency of stochastic coordinate descent for solving large-scale problems. However, they use a conservative constant stepsize during the primal-dual updates, which leads to an unsatisfying convergence rate especially for unnormalized data.

In this work, we propose an \emph{adaptive} stochastic primal-dual coordinate descent (\emph{AdaSPDC}) method for solving the Sep-CCSP problem (\ref{eq:sepccsp}), which is a non-trivial extension of SPDC. By carefully exploiting the structure of individual subproblem, we propose an adaptive stepsize rule for both primal and dual updates according to the chosen subset of coordinate blocks in each iteration. Both theoretically and empirically, we show that AdaSPDC could yield a significantly better convergence performance than SPDC and other state-of-the-art methods.

The remaining structure of the paper is as follows. Section~\ref{sec:pd} summarizes the general primal-dual framework our method and SPDC are based on. Then we elaborate our method AdaSPDC in Section~\ref{sec:adaspdc}, where both the theoretical result and its comparison with SPDC are provided. In Section~\ref{sec:exp}, we apply our method into regularized ERM tasks, and experiment with both synthetic and real-world datasets, and we show the superiority of AdaSPDC over other competitive methods empirically. Finally, Section~\ref{sec:con} concludes the work.

\section{Primal-dual Framework for Convex-Concave Saddle Point Problems}
\label{sec:pd}
Chambolle and Pock \cite{chambolle2011} proposed a first-order primal-dual method for the CCSP problem (\ref{eq:origin}). We refer this algorithm as PDCP. The update of PDCP in the $(t+1)$th iteration is  as follows:
\begin{align}
\yB^{t+1} &= \argmin_{\yB} \phi^{*}(\yB) - \langle \overline{\xB}^t,  \KB \yB \rangle + \frac{1}{2\sigma} \| \yB - \yB^t  \|_2^2 \\
\xB^{t+1} &= \argmin_{\xB} g(\xB) + \langle \xB, \KB \yB^{t+1}\rangle + \frac{1}{2\tau} \| \xB - \xB^t  \|_2^2 \\
\overline{\xB}^{t+1} &= \xB^{t+1} + \theta (\xB^{t+1} - \xB^t). \label{eq:extra_origin}
\end{align}
When the parameter configuration satisfies $\tau \sigma  \leq 1/ \| \KB \|^2$ and $\theta = 1$, PDCP could achieve $O(1/T)$ convergence rate for general convex function $\phi^*(\cdot)$ and $g(\cdot)$, where $T$ is total number of iterations. When $\phi^*(\cdot)$ and $g(\cdot)$ are both strongly convex,  a linear convergence rate can be achieved by using a more scheduled stepsize. PDCP is a batch method and non-stochastic, i.e., it has to update all the dual coordinates in each iteration for Sep-CCSP problem, which will be computationally intensive for large-scale (high-dimensional) problems.

SPDC \cite{zhang2014} can be viewed as a stochastic variant of the batch method PDCP for handling Sep-CCSP problem. However, SPDC uses a conservative constant stepsize for primal and dual updates.
Both PDCP and SPDC do not consider the structure of matrix $\KB$ and only apply constant stepsize for all coordinates of primal and dual variables. This might limit their convergence performance in reality.

Based on this observation,  we exploit the structure of matrix $\KB$ (i.e., $\frac{1}{n} \AB$) and propose an adaptive stepsize rule for efficiently solving Sep-CCSP problem. A better linear convergence rate could be yielded when $\phi_i^*(\cdot)$ and $g(\cdot)$ are strongly convex. Our algorithm will be elaborated in the following section.


\section{Adaptive Stochastic Primal-Dual Coordinate Descent}
\label{sec:adaspdc}
As a non-trivial extension of SPDC \cite{zhang2014}, our method AdaSPDC solves the Sep-CCSP problem~(\ref{eq:sepccsp}) by using an adaptive parameter configuration.  Concretely, we optimize $L(\xB, \yB)$ by alternatively updating the dual and primal variables in a principled way. Thanks to the separable structure of $\phi(\yB)$, in each iteration we can randomly select $m$ blocks of dual variables whose indices are denoted as $S_t$, and then we only update these selected blocks in the following way,
\begin{equation}
\yB_i^{t+1} =\argmin_{\yB_i} \left[\phi_i(\yB_i) - \left< \overline{\xB}^t, \AB_i \yB_i \right> + \frac{1}{2\sigma_i} \| \yB_i - \yB_i^t \|_2^2\right], \text{ if }i \in S_t.
\label{eq:dualupdate}
\end{equation}
For those coordinates in blocks not selected, $i \notin S_t$, we just keep $\yB_i^{t+1} = \yB_i^t$. By exploiting the structure of individual $\AB_i$, we configure the stepsize parameter of the proximal term $\sigma_i$ \emph{adaptively}
\begin{equation}
\sigma_i = \frac{1}{2R_i}\sqrt{\frac{n\lambda}{m\gamma}}, \label{eq:sigma}
\end{equation}
where $R_i = \| \AB_i \|_2 = \sqrt{\mu_{\text{max}} \left( \AB_i^T \AB_i \right)}$, with $\| \cdot \|_2$ is the spectral norm of a matrix and $\mu_{\max}(\cdot)$ to denote the maximum singular value of a matrix.

Our step size is different from the one used in SPDC \cite{zhang2014}, where $R$ is a constant $R = \max \{\| \aB_i \|_2: i = 1,\dots,n \}$ (since SPDC only considers ERM problem, the matrix $\AB_i$ is a feature vector $\aB_i$).  

\emph{Remark}. Intuitively, $R_i$ in AdaSPDC can be understood as the coupling strength between the $i$-th dual variable block and primal variable, measured by the spectral norm of matrix $\AB_i$. Smaller coupling strength allows us to use larger stepsize for the current dual variable block without caring too much about its influence on primal variable, and vice versa.  Compared with SPDC, our proposal of an adaptive coupling strength for the chosen coordinate block directly results in larger step size, and thus helps to improve convergence speed. 

In the stochastic dual update, we also use an intermediate variable $\overline{\xB}^t$ as in PDCP algorithm, and we will describe its update later.

Since we assume $g(\xB)$ is not separable, we update the primal variable as a whole,
\begin{equation}
\xB^{t+1} = \argmin_{\xB} \left[g(\xB) + \left< \xB, \rB^t + \frac{1}{m} \sum_{j \in S_t} \AB_j(\yB_j^{t+1} - \yB_j^t) \right> + \frac{1}{2\tau^t} \| \xB - \xB^t \|_2^2\right]. \label{eq:primalupdate}
\end{equation}
The proximal parameter $\tau^t$ is also configured \emph{adaptively},
\begin{equation}
\tau^t = \frac{1}{2R_{\max}^t}\sqrt{\frac{m\gamma}{n\lambda}}, \label{eq:tau}
\end{equation}
where $R_{\max}^t = \max \left\{ R_i | i \in S_t \right\}$, compared with constant $R$ used in SPDC.
To account for the incremental change after the latest dual update, an auxiliary variable $\rB^t = \frac{1}{n}\sum_{i=1}^n \AB_i \yB_i^t$ is used and  updated as follows
   \begin{equation}
   \rB^{t+1} = \rB^t + \frac{1}{n}\sum_{j \in S_t}  \AB_j \left( \yB_j^{t+1} -  \yB^{t}_j \right).
   \label{eq:rupdate}
   \end{equation}
Finally, we update the intermediate variable $\overline{\xB}$, which implements an extrapolation step over the current $\xB^{t+1}$ and can help to provide faster convergence rate \cite{nesterov2004,chambolle2011}.
\begin{equation}
\overline{\xB}^{t+1} = \xB^{t+1} + \theta^t (\xB^{t+1} - \xB^t), \label{eq:extrapolation}
\end{equation}
where $\theta^t$ is configured adaptively as
\begin{equation}
\theta^t = 1- \frac{1}{n/m + R^t_{\max}\sqrt{(n/m)/(\lambda \gamma)}}, \label{eq:theta}
\end{equation}
which is contrary to the constant $\theta$ used in SPDC.

\begin{algorithm}[tb]
   \caption{AdaSPDC for Separable Convex-Concave Saddle Point Problems}
   \label{alg:adaspdc}
\begin{algorithmic}[1]
   \STATE {\bfseries Input:} number of blocks picked in each iteration $m$ and number of iterations $T$.
   \STATE {\bfseries Initialize:} $\xB^0$, $\yB^0$, $\overline{\xB}^0 = \xB^0$, $\rB^0 = \frac{1}{n}\sum_{i=1}^n \AB_i \yB_i^0$
   \FOR{$t=0,1,\ldots, T-1$}
   \STATE Randomly pick $m$ dual coordinate blocks from $\{ 1, \dots, n \}$ as indices set $S_t$, with the probability of each block being selected equal to $m/n$.
   \STATE According to the selected subset $S_t$, compute the adaptive parameter configuration of  $\sigma_i$, $\tau^t$ and $\theta^t$ using Eq.~(\ref{eq:sigma}), (\ref{eq:tau}) and (\ref{eq:theta}), respectively.
   \FOR{each selected block in parallel}
   \STATE Update the dual variable block using Eq.(\ref{eq:dualupdate}).
   \ENDFOR
   \STATE Update primal variable using Eq.(\ref{eq:primalupdate}).
   \STATE Extrapolate primal variable block using Eq.(\ref{eq:extrapolation}).
   \STATE Update the auxiliary variable $\rB$ using Eq.(\ref{eq:rupdate}).

   \ENDFOR
\end{algorithmic}
\end{algorithm}

The whole procedure for solving Sep-CCSP problem (\ref{eq:sepccsp}) using AdaSPDC is summarized in Algorithm \ref{alg:adaspdc}. There are several notable characteristics of our algorithms.
\begin{itemize}
\item Compared with SPDC, our method uses adaptive step size to obtain faster convergence (will be shown in Theorem \ref{th:main}), while the whole algorithm does not bring any other extra computational complexity. As demonstrated in the experiment Section~\ref{sec:exp}, in many cases, AdaSPDC provides significantly better performance than SPDC.
\item Since, in each iteration, a number of block coordinates can be chosen and updated independently (with independent evaluation of individual step size), this directly enables parallel processing, and hence use on modern computing clusters. The ability to select an arbitrary number of blocks can help to make use of all the computation structure available as effectively as possible.
\end{itemize}

\subsection{Convergence Analysis}
We characterise the convergence performance of our method in the following theorem.
\begin{theorem}
Assume that each $\phi^*_{i}(\cdot)$ is $\gamma$-strongly convex, and $g(\cdot)$ is $\lambda$-strongly convex,
and given the parameter configuration in Eq.~(\ref{eq:sigma}), (\ref{eq:tau}) and (\ref{eq:theta}), then after $T$ iterations in Algorithm~\ref{alg:adaspdc}, the algorithm achieves the following convergence performance
\begin{align}
&\left( \frac{1}{2\tau^T} + \lambda  \right) \Ebb \left[ \| \xB^T - \xB^{\star} \|_2^2 \right] + \Ebb \left[ \| \yB^T - \yB^{\star} \|_{\nuB}^2 \right] \nonumber \\
\leq & \left( \prod_{t=1}^T  \theta^t\right) \left( \left( \frac{1}{2\tau^T} + \lambda  \right) \| \xB^0 - \xB^{\star} \|_2^2   +  \| \yB^0 - \yB^{\star} \|_{\nuB '}^2 \right),
\end{align}
where $(\xB^{\star}, \yB^{\star})$ is the optimal saddle point, $\nu_i = \frac{1/(4\sigma_i) + \gamma}{m} $, $\nu_i ' = \frac{1/(2\sigma_i) + \gamma}{m}$, and $\| \yB^T - \yB^{\star} \|_{\nuB}^2 = \sum_{i=1}^n \nu_i \|\yB^T_i - \yB^{\star}_i \|_2^2 $. \label{th:main}
\end{theorem}
Since the proof of the above is technical, we provide it in the Supplementary Material.

In our proof, given the proposed parameter $\theta^t$, the critical point for obtaining a sharper linear convergence rate than SPDC is that we configure $\tau^t$ and $\sigma_i$ as Eq.~(\ref{eq:tau}) and (\ref{eq:sigma}) to guarantee the positive definiteness of the following matrix in the $t$-th iteration,
\begin{equation}
\PB =
\begin{bmatrix}
\frac{m}{2\tau^t} \IB & -\AB_{S_t} \\
-\AB_{S_t}^T &  \frac{1}{2\diag (\sigmaB_{S_t})}
\end{bmatrix},
\end{equation}
where $\AB_{S_t} = [\dots, \AB_i, \dots] \in \Rbb^{d \times mq_i}$ and $\diag (\sigmaB_{S_t}) = \diag (\dots, \sigma_i \IB_{q_i}, \dots)\text{ for }i\in S_t$. However, we found that the parameter configuration to guarantee the  positive definiteness of $\PB$ is not unique, and there exist other valid parameter configurations besides the proposed one in this work. We leave the further investigation on other potential parameter configurations as future work.   

\subsection{More Comparison with SDPC}
Compared with SPDC \cite{zhang2014}, AdaSPDC follows the similar primal-dual framework. The crucial difference between them is that AdaSPDC proposes a larger stepsize for both dual and primal updates, see Eq.~(\ref{eq:sigma}) and (\ref{eq:tau}) compared with SPDC's parameter configuration given in Eq.(10) in \cite{zhang2014}, where SPDC applies a large constant $R = \max \{\| \aB_i \|_2: i = 1,\dots,n \}$ while AdaSPDC uses a more adaptive value of $R_i$ and $R_{\max}^t$ for $t$-th iteration to account for the different coupling strength between the selected dual coordinate block and primal variable. This difference directly means that AdaSPDC can potentially obtain a significantly sharper linear convergence rate than SPDC, since the decay factor $\theta^t$ of AdaSPDC is smaller than $\theta$ in SPDC (Eq.(10) in \cite{zhang2014}) , see Theorem~\ref{th:main} for AdaSPDC compared with SPDC (Theorem~1 in \cite{zhang2014}). The empirical performance of the two algorithms will be demonstrated in the experimental Section~\ref{sec:exp}.

To mitigate the problem that SPDC uses a large $R$, the authors of SPDC proposes to non-uniformly sample the the dual coordinate to update in each iteration according to the norm of the each $\aB_i$. However, as we show later in the empirical experiments, this non-uniform sampling does not work very well for some datasets. By configuring the adaptive stepsize explicitly, our method AdaSPDC provides a better solution for unnormalized data compared with SPDC, see Section~\ref{sec:exp} for more empirical evidence.

Another difference is that SPDC only considers the regularized ERM task, i.e., only handling the case that each $\AB_i$ is a feature vector $\aB_i$, while AdaSPDC extends that $\AB_i$ can be a matrix so that AdaSPDC can cover a wider range of applications than SPDC, i.e. in each iteration, a number of \emph{block} coordinates could be selected while for SPDC only a number of coordinates are allowed.

\begin{figure}[h]
\vskip -0.3in
\begin{center}
\begin{tabular}{cc}
\includegraphics[width=0.4\columnwidth]{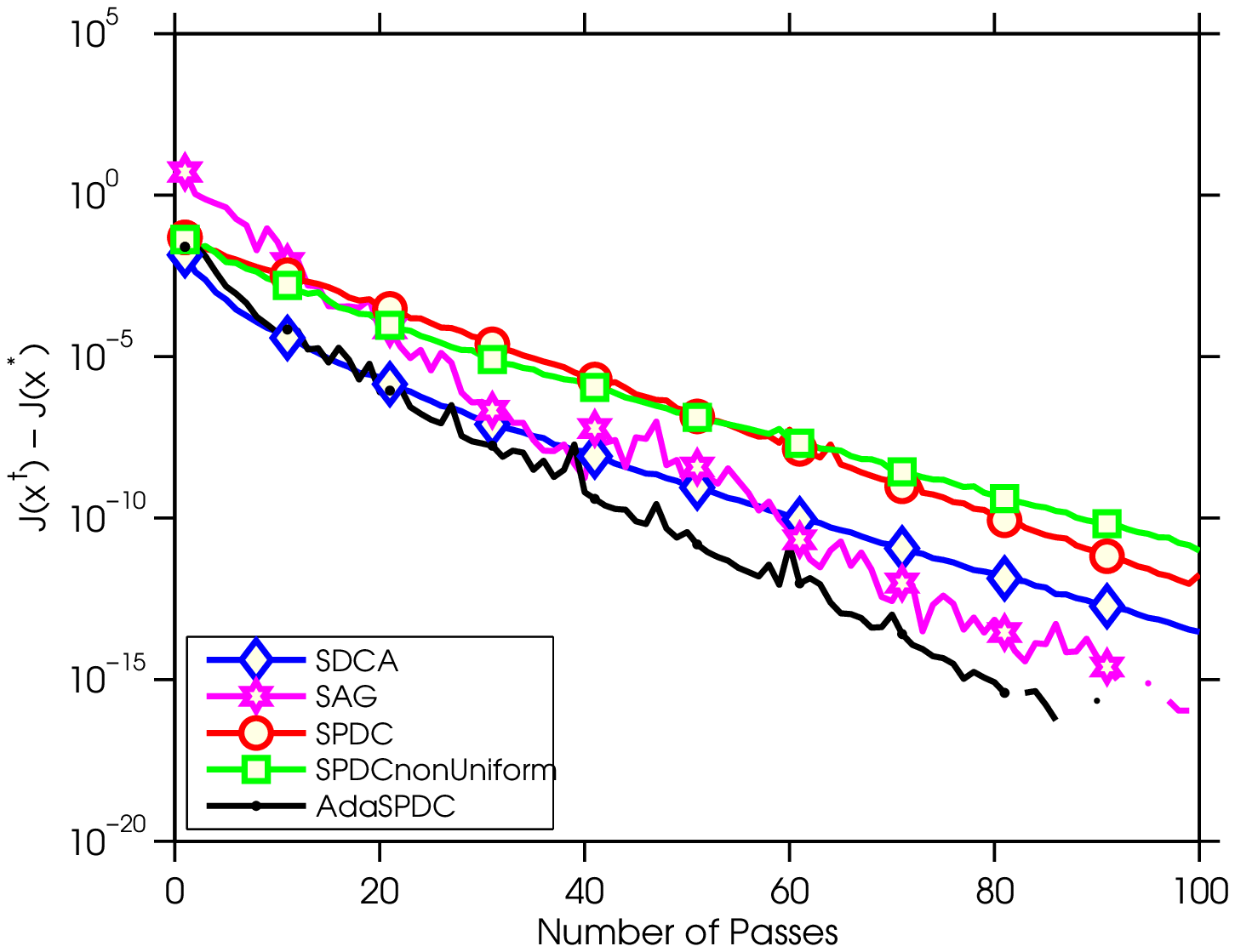} & \includegraphics[width=0.4\columnwidth]{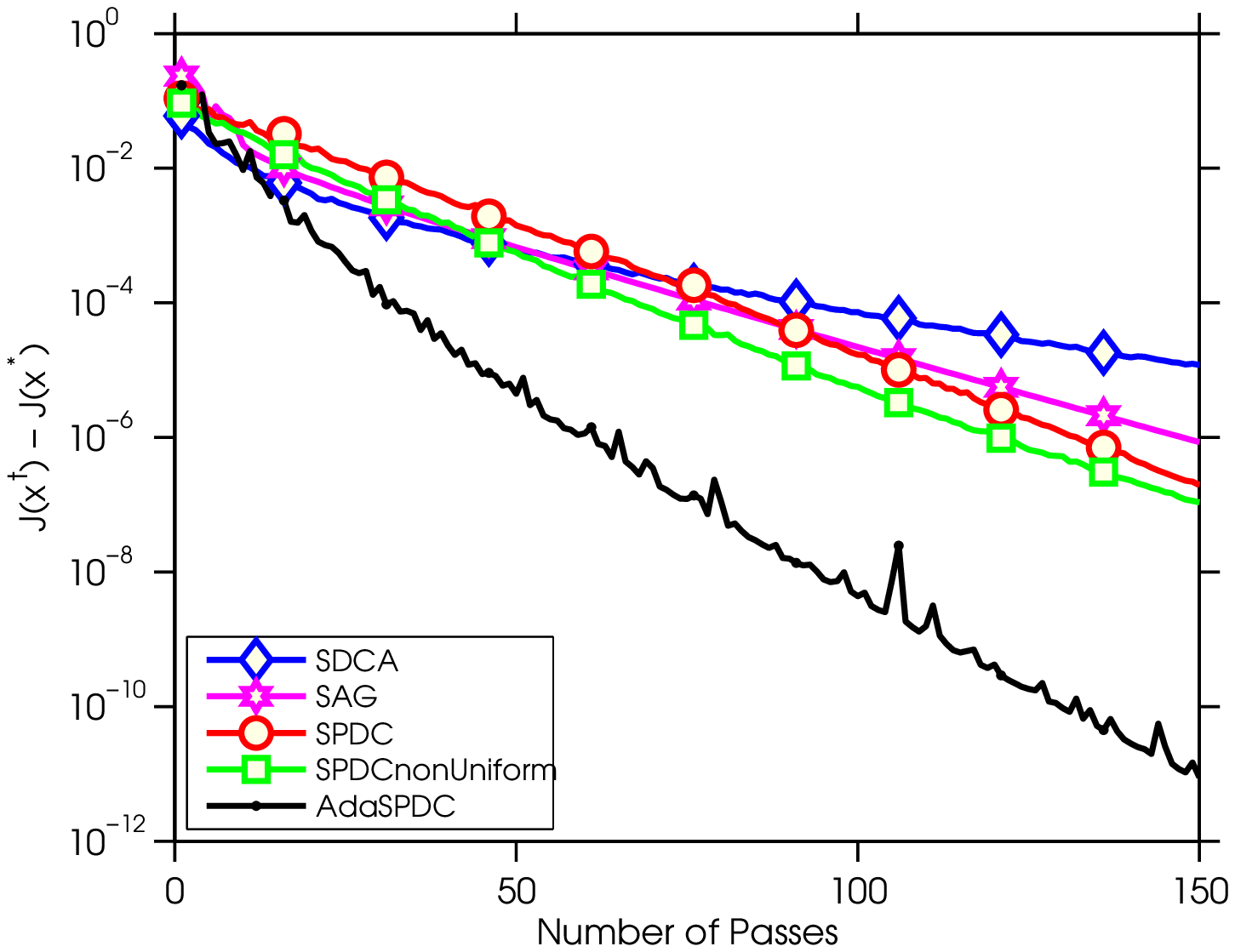} \\
(a) $\lambda = 10^{-3}$ & (b) $\lambda = 10^{-4}$  \\
\includegraphics[width=0.4\columnwidth]{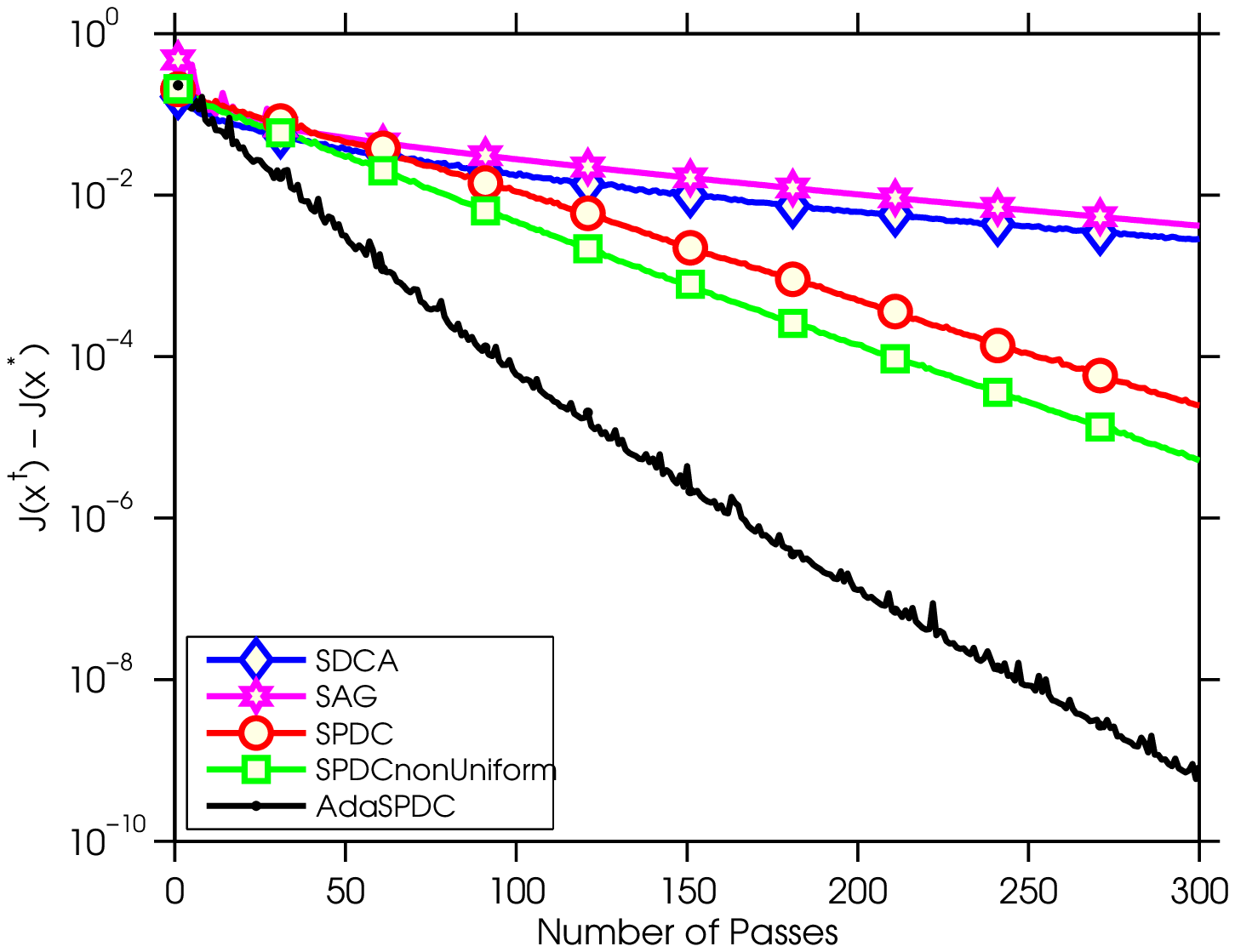}  & \includegraphics[width=0.4\columnwidth]{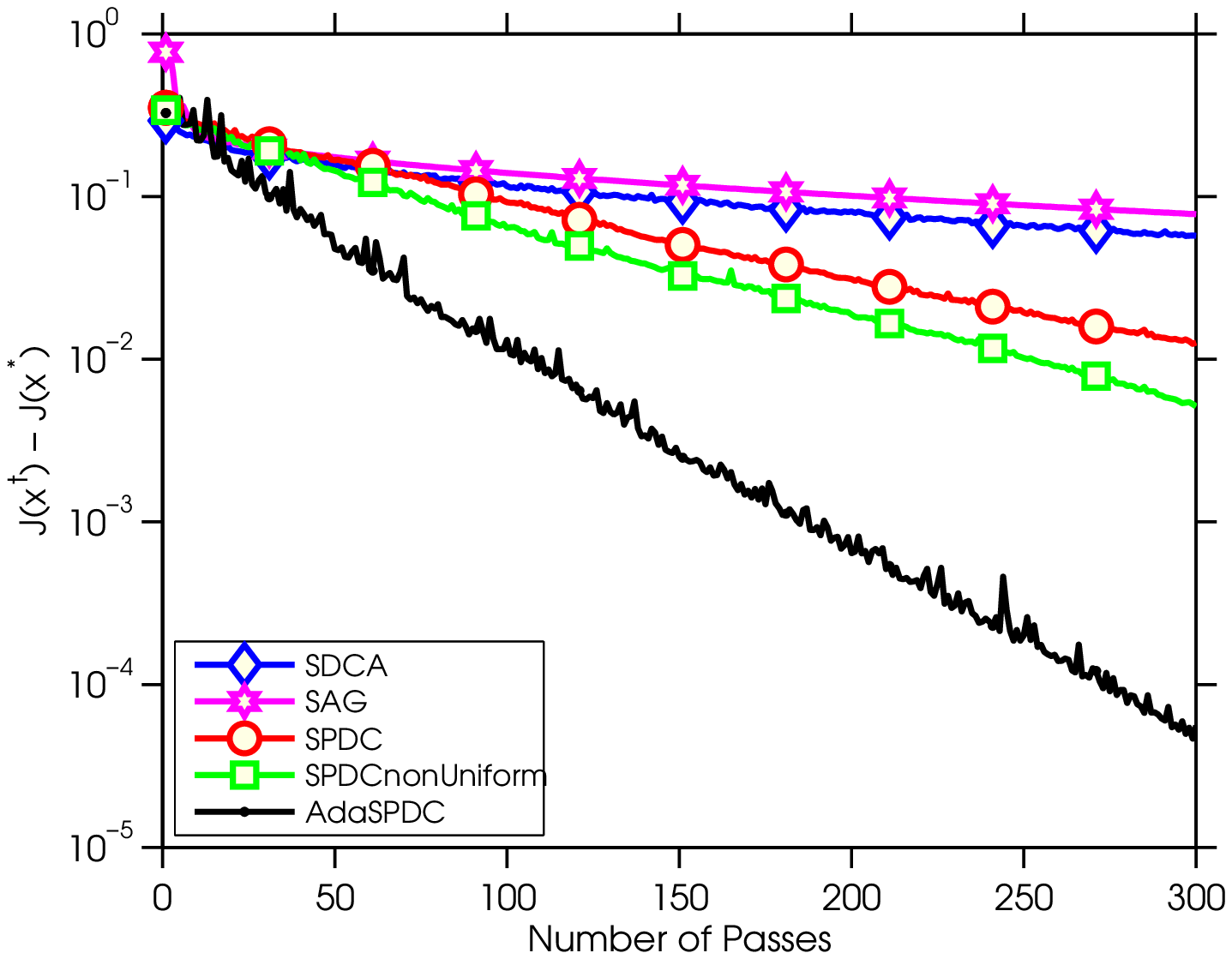} \\
(c) $\lambda = 10^{-5}$ & (d) $\lambda = 10^{-6}$
\end{tabular}
\vspace{-6mm}
\end{center}
\caption{\small Ridge regression with synthetic data: comparison of convergence performance w.r.t. the number of passes. Problem size: $d=1000, n =1000$. We evaluate the convergence performance using objective suboptimality, $J(\xB^t) - J(\xB^{\star})$.
\label{fig:ridge}}
\vspace{-3mm}
\end{figure}

\section{Empirical Results}
\label{sec:exp}
In this section, we appy AdaSPDC to several regularized empirical risk minimization problems. The experiments are conducted to compare our method AdaSPDC with other competitive stochastic optimization methods,  including SDCA \cite{shalev2013}, SAG \cite{schmidt2013}, SPDC with uniform sampling and non-uniform sampling \cite{zhang2014}. In order to provide a fair comparison with these methods, in each iteration only one dual coordinate (or data instance) is chosen, i.e., we run all the methods sequentially. To obtain results that are independent of the practical implementation of the algorithm, we measure the algorithm performance in term of objective suboptimality w.r.t. the effective passes to the entire data set.

Each experiment is run 10 times and the average results are reported to show statistical consistency. We present all the experimental results we have done for each application.

\subsection{Ridge Regression}
We firstly apply our method AdaSPDC into a simple ridge regression problem with synthetic data. The data is generated in the same way as Zhang and Xiao \cite{zhang2014}; $n=1000$ i.i.d. training points $\{\aB_i, b_i \}_{i=1}^n$ are generated in the following manner,
\begin{equation*}
b = \aB^T \xB^{\diamond} + \epsilon, \quad \aB \sim \Ncal (\zeroB, \SigmaB), \quad \epsilon \sim \Ncal (0, 1),
\end{equation*}
where $\aB \in \Rbb^d$ and $d =1000$, and the elements of the vector $\xB^{\diamond}$ are all ones. The covariance matrix $\SigmaB$ is set to be diagonal with $\Sigma_{jj} = j^{-2}$, for $j = 1, \dots, d$.  Then the ridge regression tries to solve the following optimization problem,
\begin{equation}
\min_{\xB \in \Rbb^d} \left\{ J(\xB) = \frac{1}{n} \sum_{i=1}^n \frac{1}{2} (\aB_i^T \xB - b_i)^2 + \frac{\lambda}{2} \| \xB \|_2^2 \right\}.
\end{equation}
The optimal solution of the above ridge regression can be found as
\begin{equation*}
\xB^{\star} = \left( \AB \AB^T + n\lambda \IB_{d }\right)^{-1} \AB \bB.
\end{equation*}
By employing the conjugate dual of quadratic loss (crossref, Eq.~(\ref{eq:conjugate})), we can reformulate the ridge regression as the following Sep-CCSP problem,
\begin{equation}
\min_{\xB \in \Rbb^d} \max_{\yB \in \Rbb^n} \frac{\lambda}{2} \| \xB \|_2^2 + \frac{1}{n}\sum_{i=1}^n \left( \left<  \xB, y_i \aB_i \right> -  \left( \frac{1}{2} y_i^2 + b_i y_i  \right) \right).
\end{equation}
It is easy to figure out that $g(\xB) = \lambda / 2 \| \xB \|_2^2$ is $\lambda$-strongly convex, and $\phi_i^*(y_i)= \frac{1}{2} y_i^2 + b_i y_i$ is $1$-strongly convex.

Thus, for ridge regression, the dual update in Eq.~(\ref{eq:dualupdate}) and primal update in Eq.~(\ref{eq:primalupdate}) of AdaSPDC have closed form solutions as below,
\begin{align*}
y_i^{t+1} &= \frac{1}{1 + 1/\sigma_i} \left( \left< \overline{\xB}^t, \aB_i \right> + b_i +  \frac{1}{\sigma_i} y_i \right), \text{ if } i \in S_t \\
\xB^{t+1} &=  \frac{1}{\lambda + 1/\tau^t} \left( \frac{1}{\tau^t} \xB^t - \left( \rB^t + \frac{1}{m} \sum_{j \in S_t} \aB_j(y_j^{t+1} - y_j^t) \right)    \right)
\end{align*}

The algorithm performance is evaluated in term of objective suboptimality (measured by $J(\xB^t) - J(\xB^{\star})$) w.r.t. number of effective passes to the entire datasets. Varying values of regularization parameter $\lambda$ are experimented to demonstrate algorithm performance with different degree of ill-conditioning, $\lambda = \{10^{-3}, 10^{-4}, 10^{-5}, 10^{-6} \}$.

Fig.~\ref{fig:ridge} shows algorithm performance with different degrees of regularization. It is easy to observe that AdaSPDC converges substantially faster than other compared methods, particularly for ill-conditioned problems. Compared with SPDC and its variant with non-uniform sampling, the usage of adaptive stepsize in AdaSPDC significantly improves convergence speed. For instance, in the case with $\lambda = 10^{-6}$, AdaSPDC achieves 100 times better suboptimality than both SPDC and its variant SPDC with non-uniform sampling after 300 passes.

\begin{table}[ht!]
\caption{\small{Benchmark datasets used in our experiments for binary classification.}
}
\centering
\small{
\begin{tabular}{c|c|c|c}
\hline
Datasets    & Number of samples  & Number of features  & Sparsity \\
\hline
\emph{w8a}       & 49,749       & 300      &            3.9\% \\
\hline
\emph{covertype} & 20,242       & 47,236   &            0.16\% \\
\hline
\emph{url} &  2,396,130       &  3,231,961   &            0.0018\% \\
\hline
\emph{quantum}   & 50,000       & 78       &        43.44\% \\
\hline
\emph{protein}   & 145,751       & 74       &        99.21\% \\
\hline
\end{tabular}}
\label{tab:datasets}
\end{table}

\begin{figure*}[ht!]
\vskip -0.1in
{\tiny
\begin{center}
\begin{tabular}{c | ccc}
  Dataset & $\lambda=10^{-5}$ & $\lambda=10^{-6}$ & $\lambda=10^{-7}$\\
\hline
w8a & \raisebox{-.5\totalheight}{\includegraphics[width=0.3\textwidth]{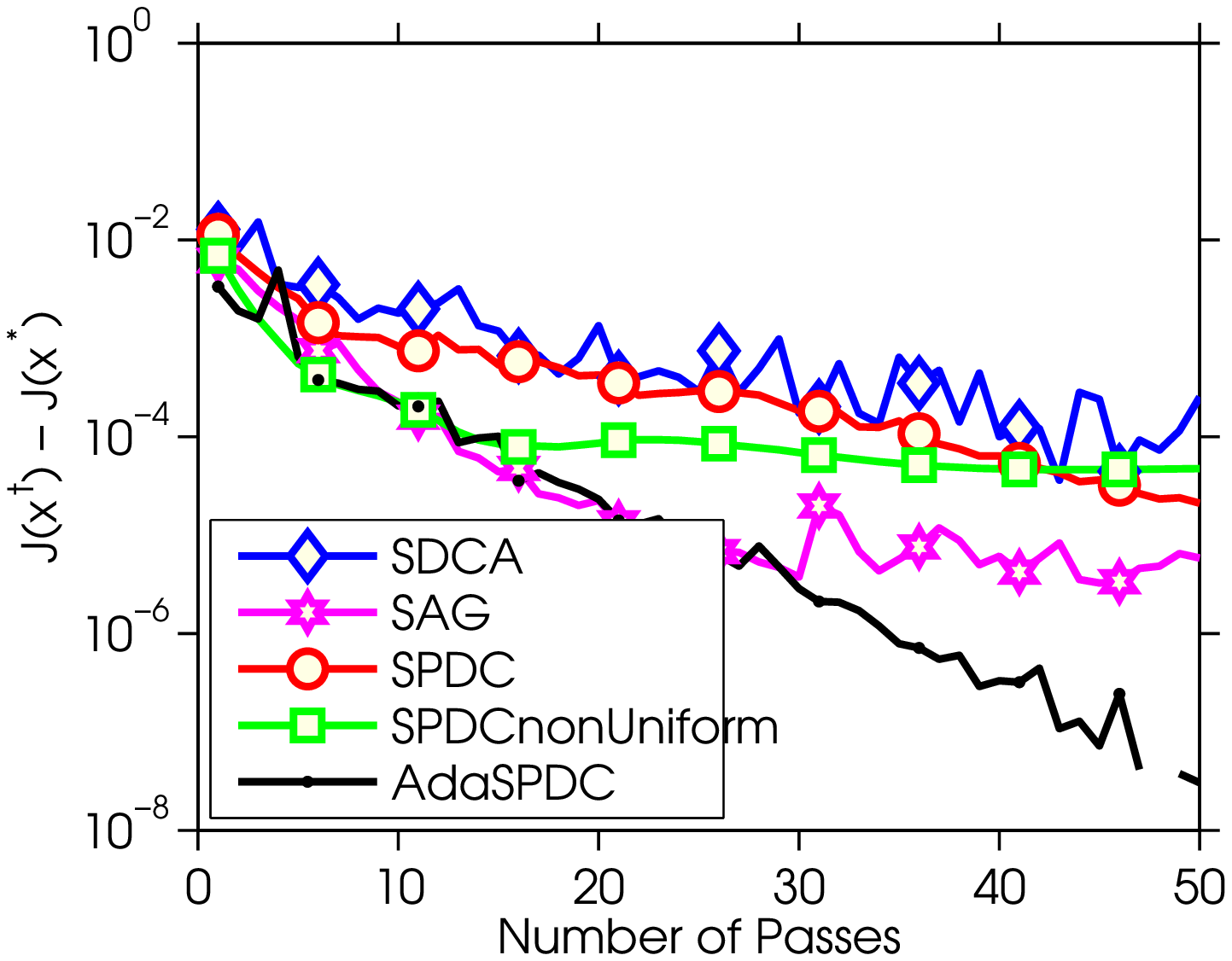}} & \raisebox{-.5\totalheight}{\includegraphics[width=0.3\textwidth]{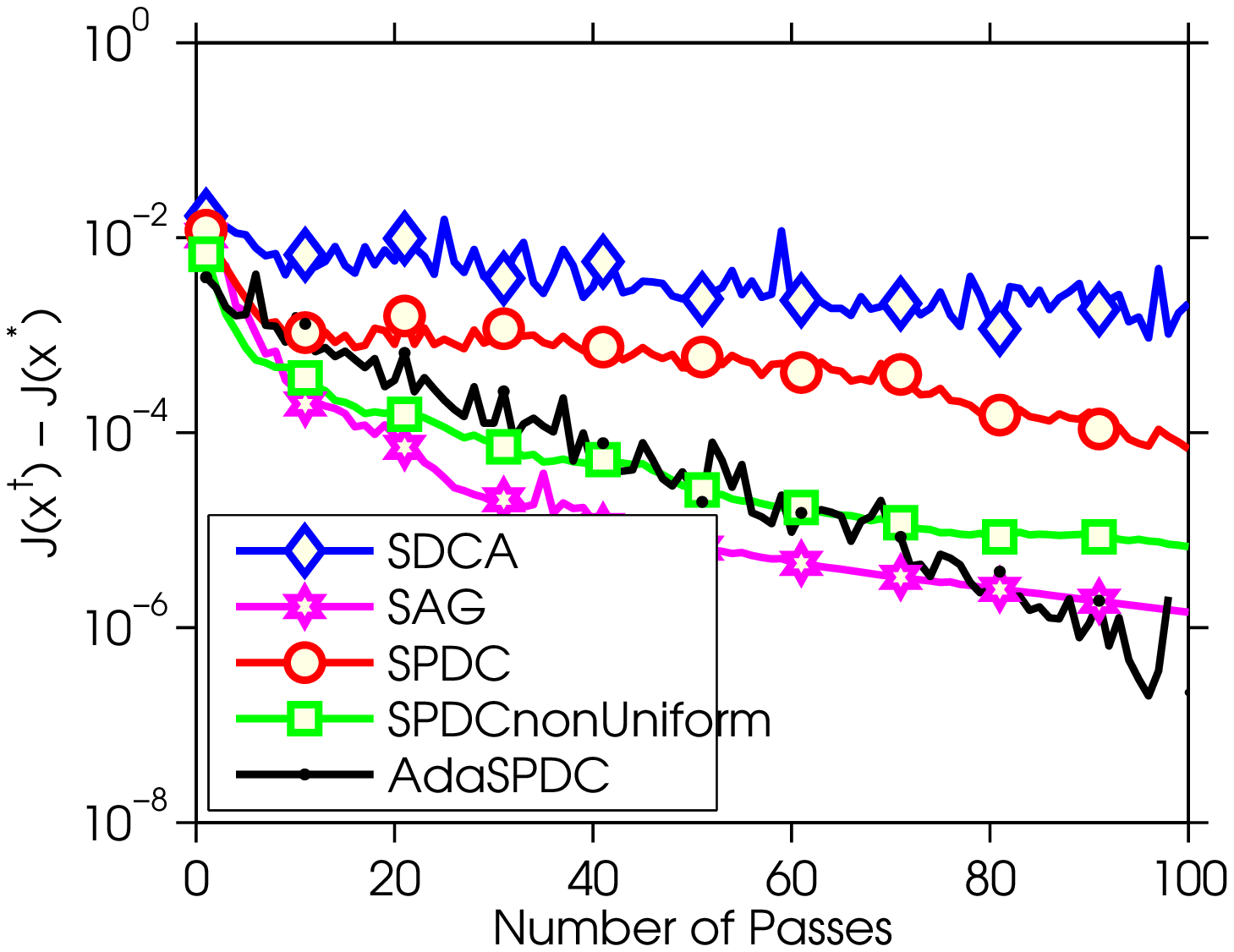}} & \raisebox{-.5\totalheight}{\includegraphics[width=0.3\textwidth]{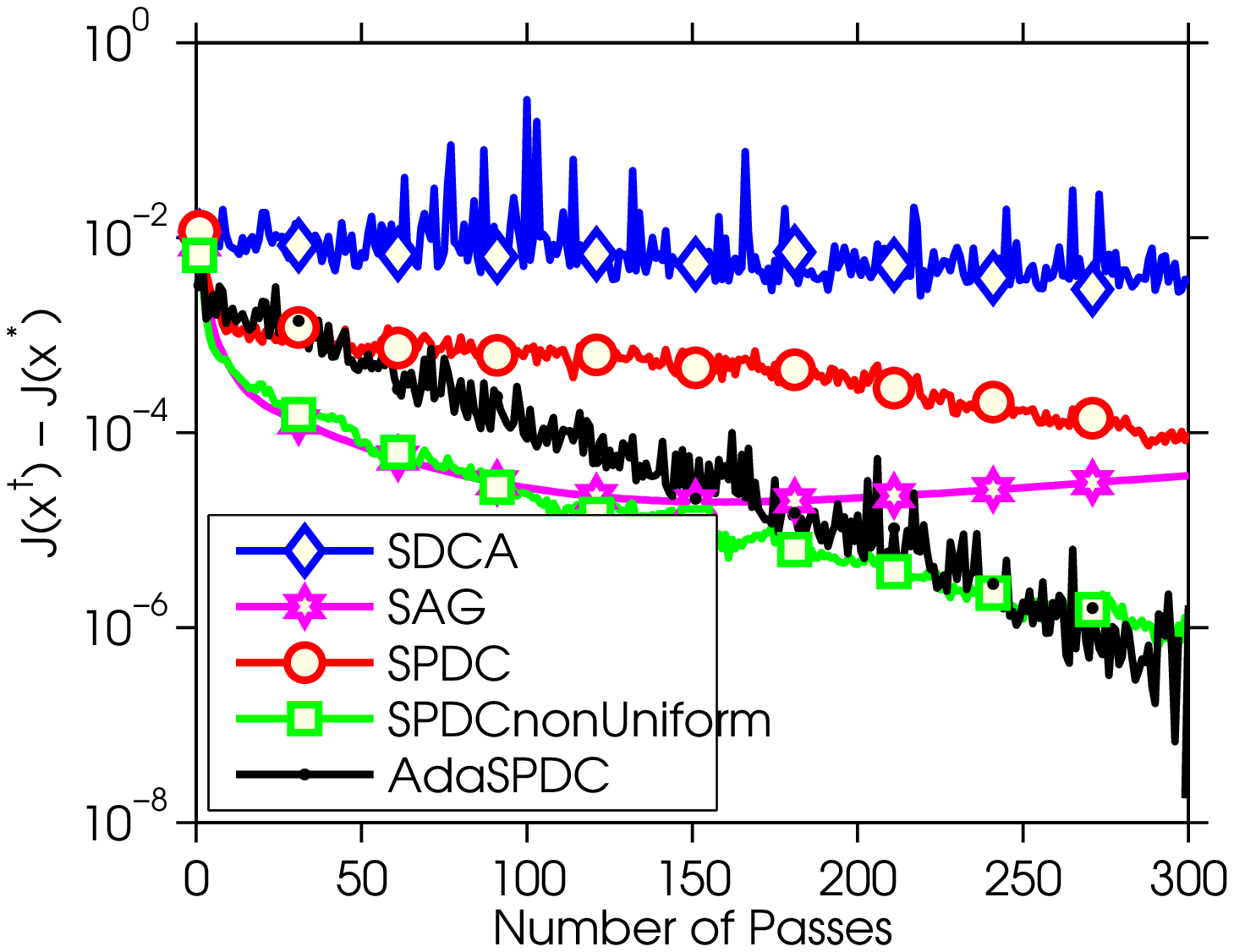}}\\
\hline
covtype & \raisebox{-.5\totalheight}{\includegraphics[width=0.3\textwidth]{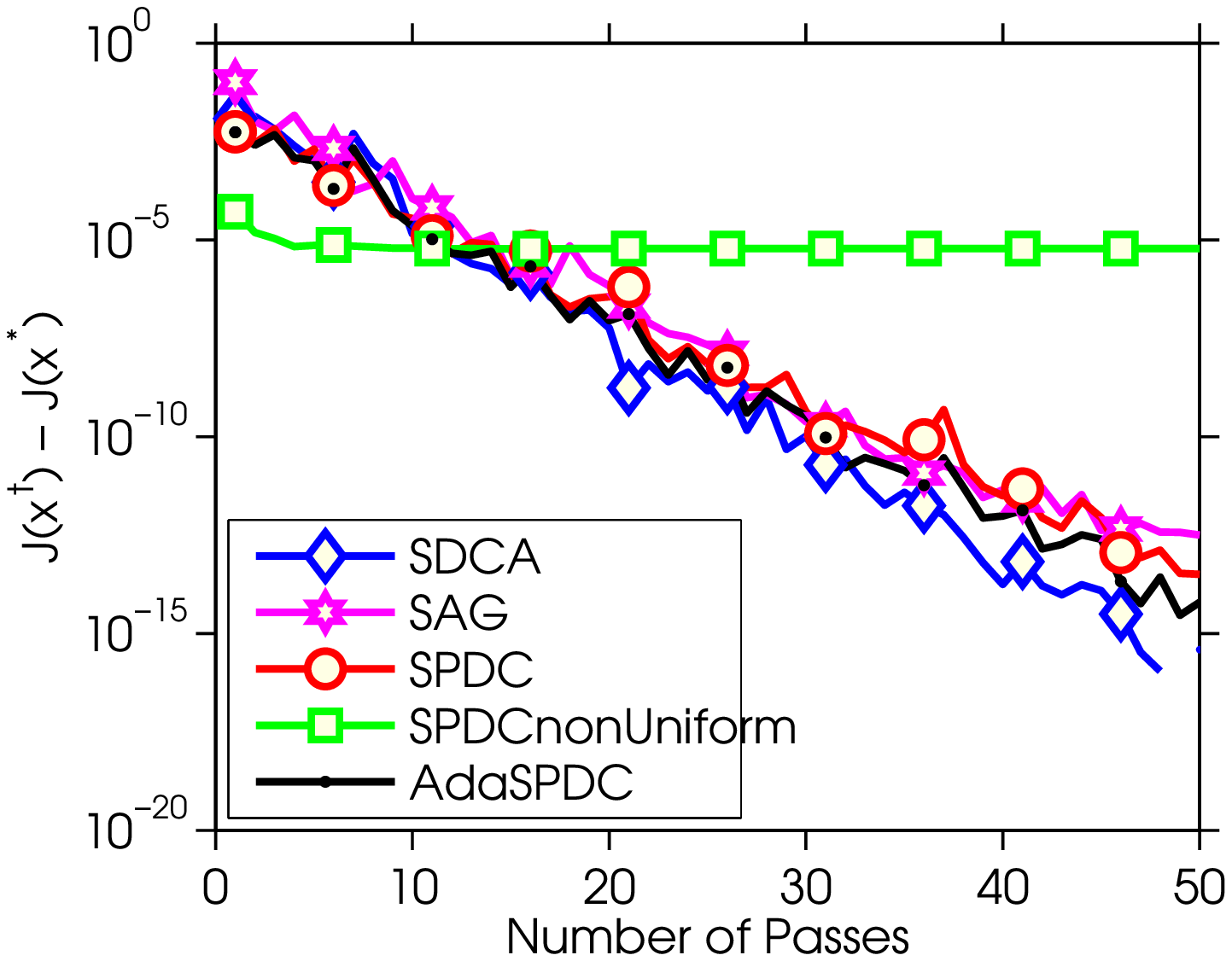}} & \raisebox{-.5\totalheight}{\includegraphics[width=0.3\textwidth]{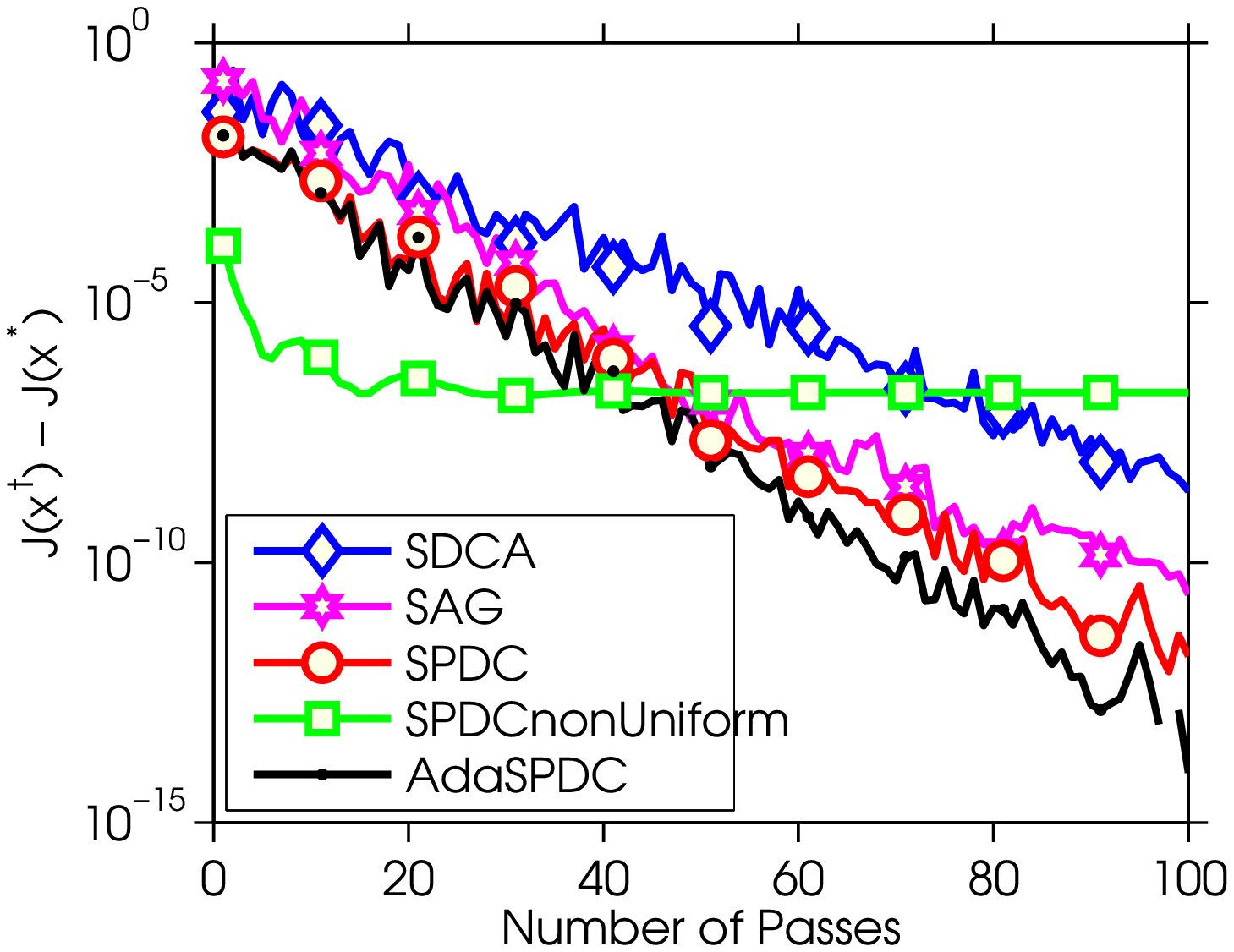}} & \raisebox{-.5\totalheight}{\includegraphics[width=0.3\textwidth]{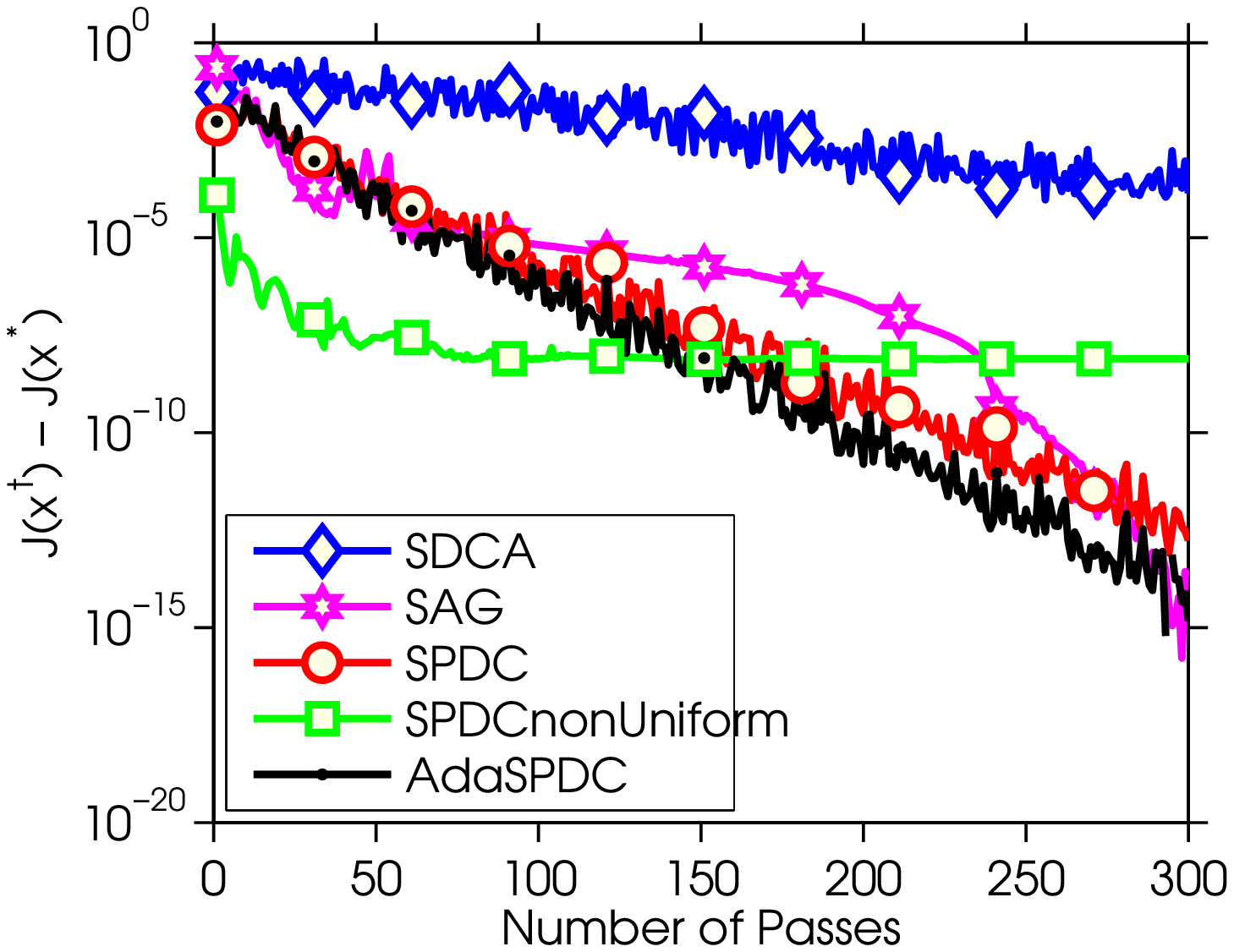}}\\
\hline
url & \raisebox{-.5\totalheight}{\includegraphics[width=0.3\textwidth]{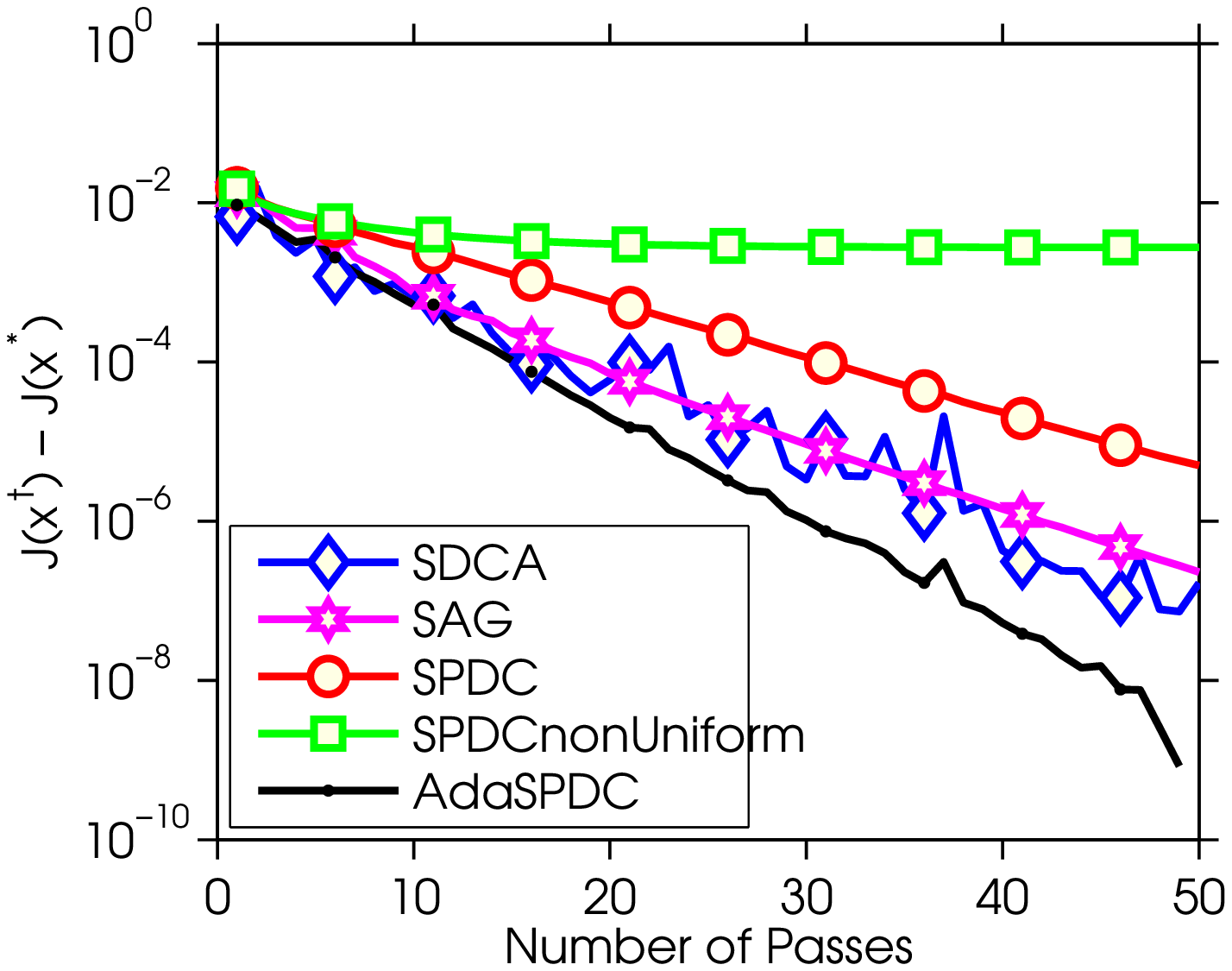}} & \raisebox{-.5\totalheight}{\includegraphics[width=0.3\textwidth]{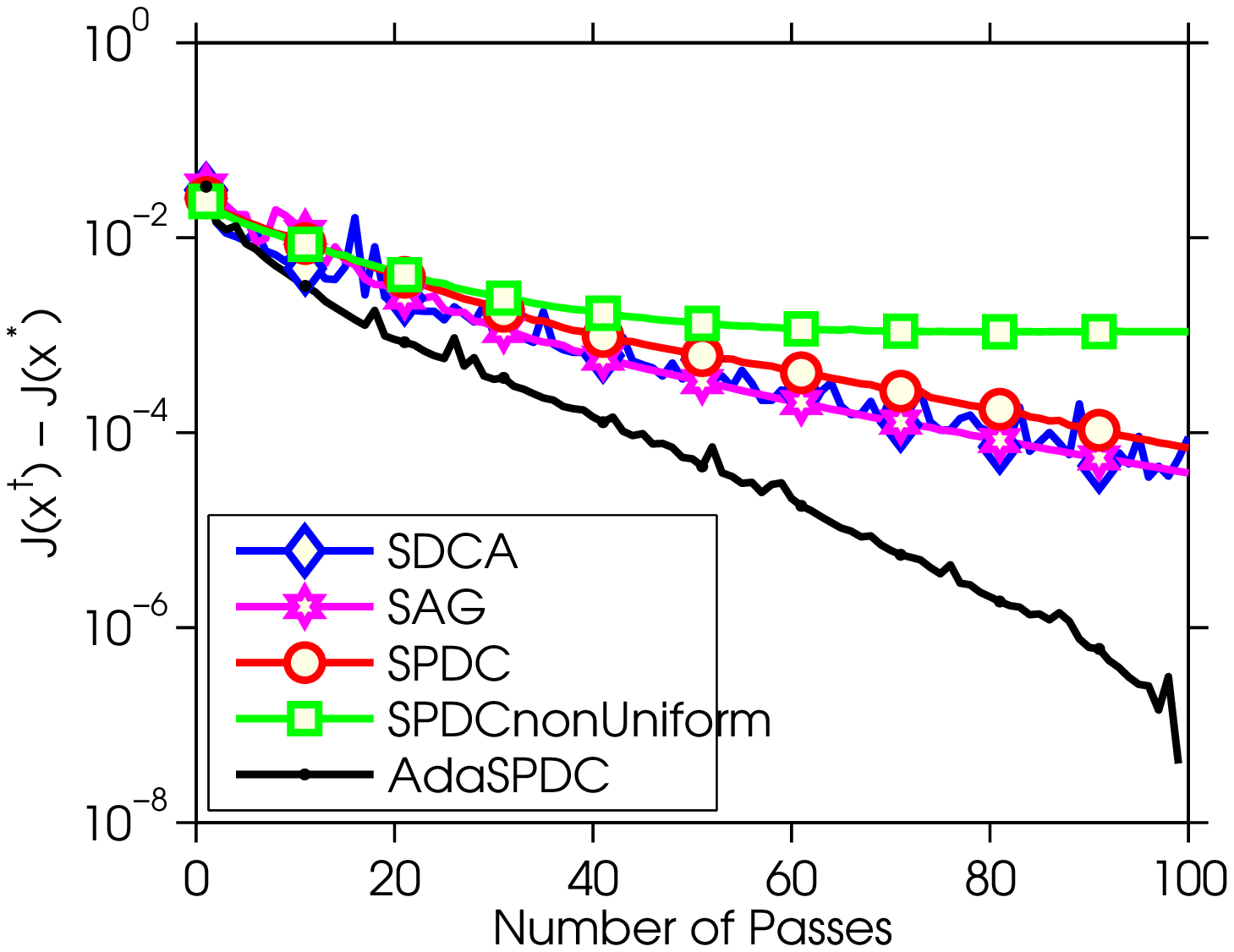}} & \raisebox{-.5\totalheight}{\includegraphics[width=0.3\textwidth]{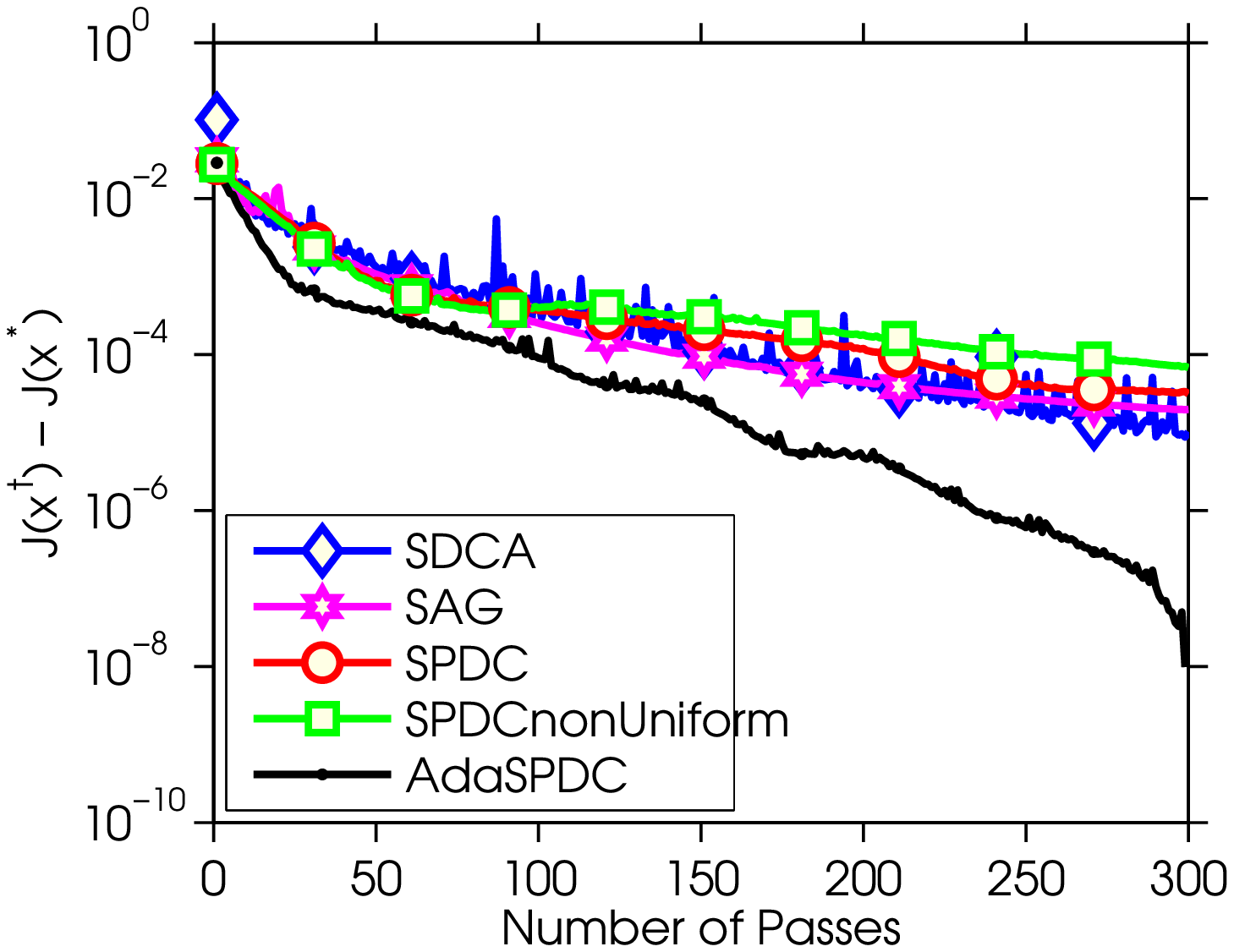}}\\
\hline
quantum & \raisebox{-.5\totalheight}{\includegraphics[width=0.3\textwidth]{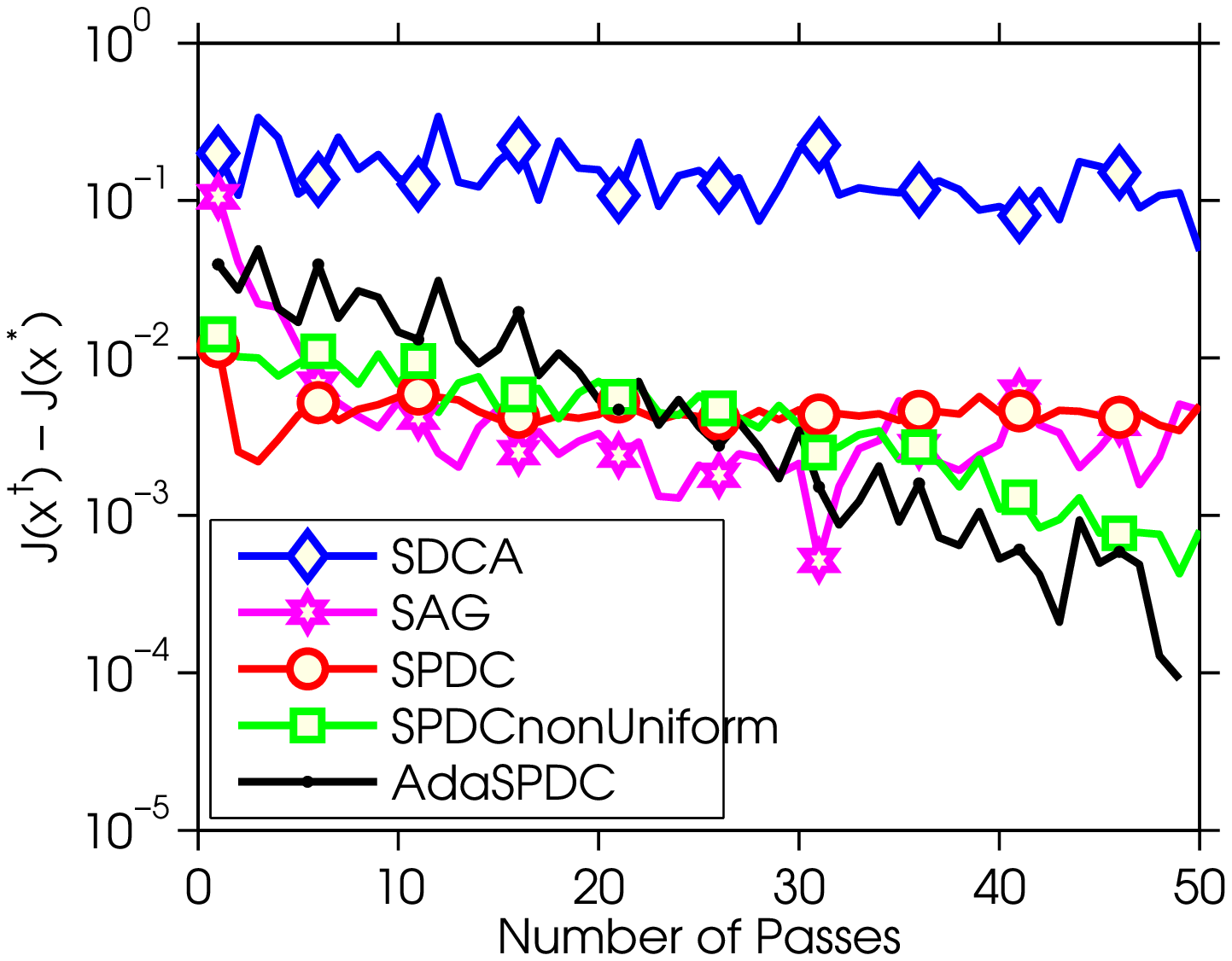}}&\raisebox{-.5\totalheight}{\includegraphics[width=0.3\textwidth]{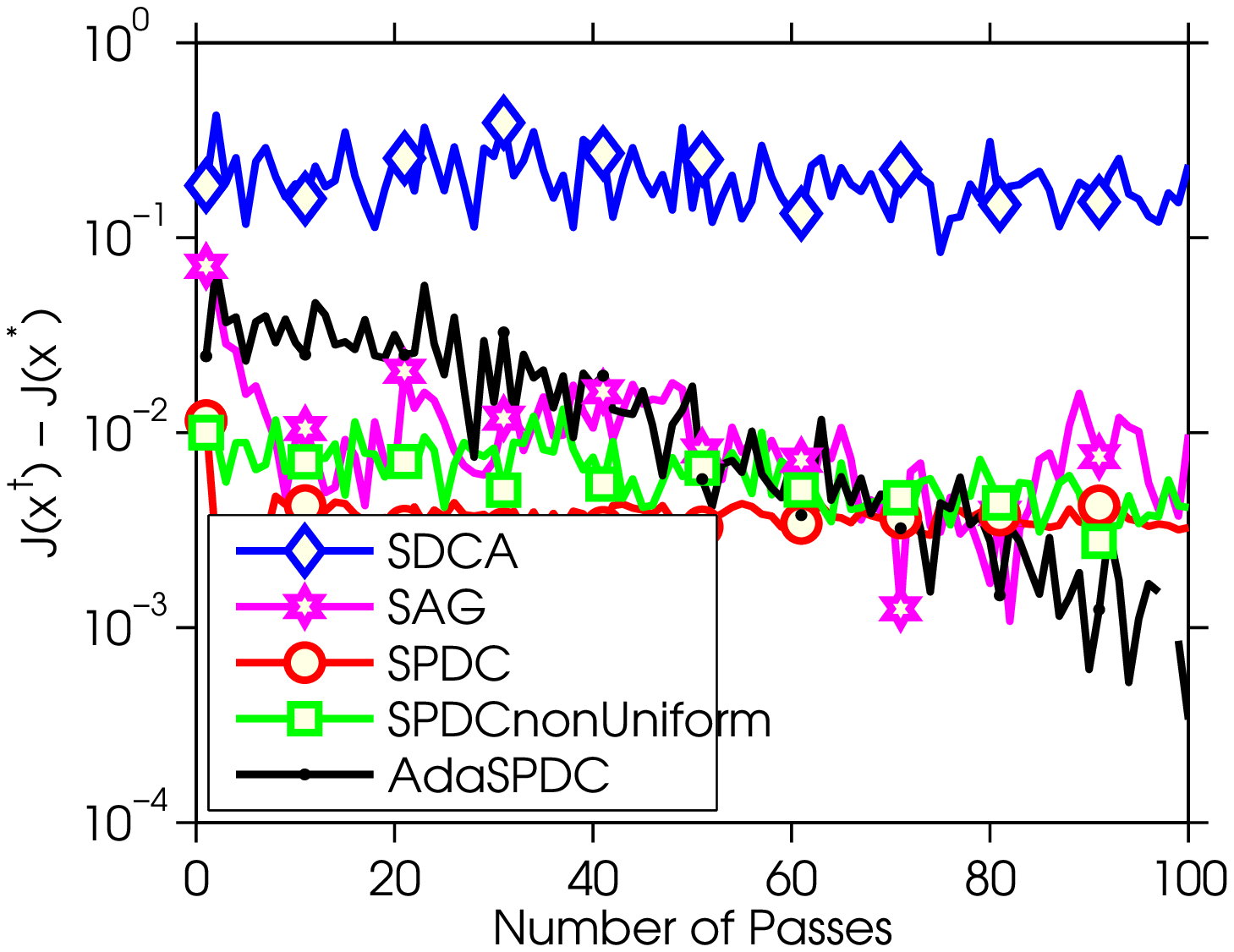}} & \raisebox{-.5\totalheight}{\includegraphics[width=0.3\textwidth]{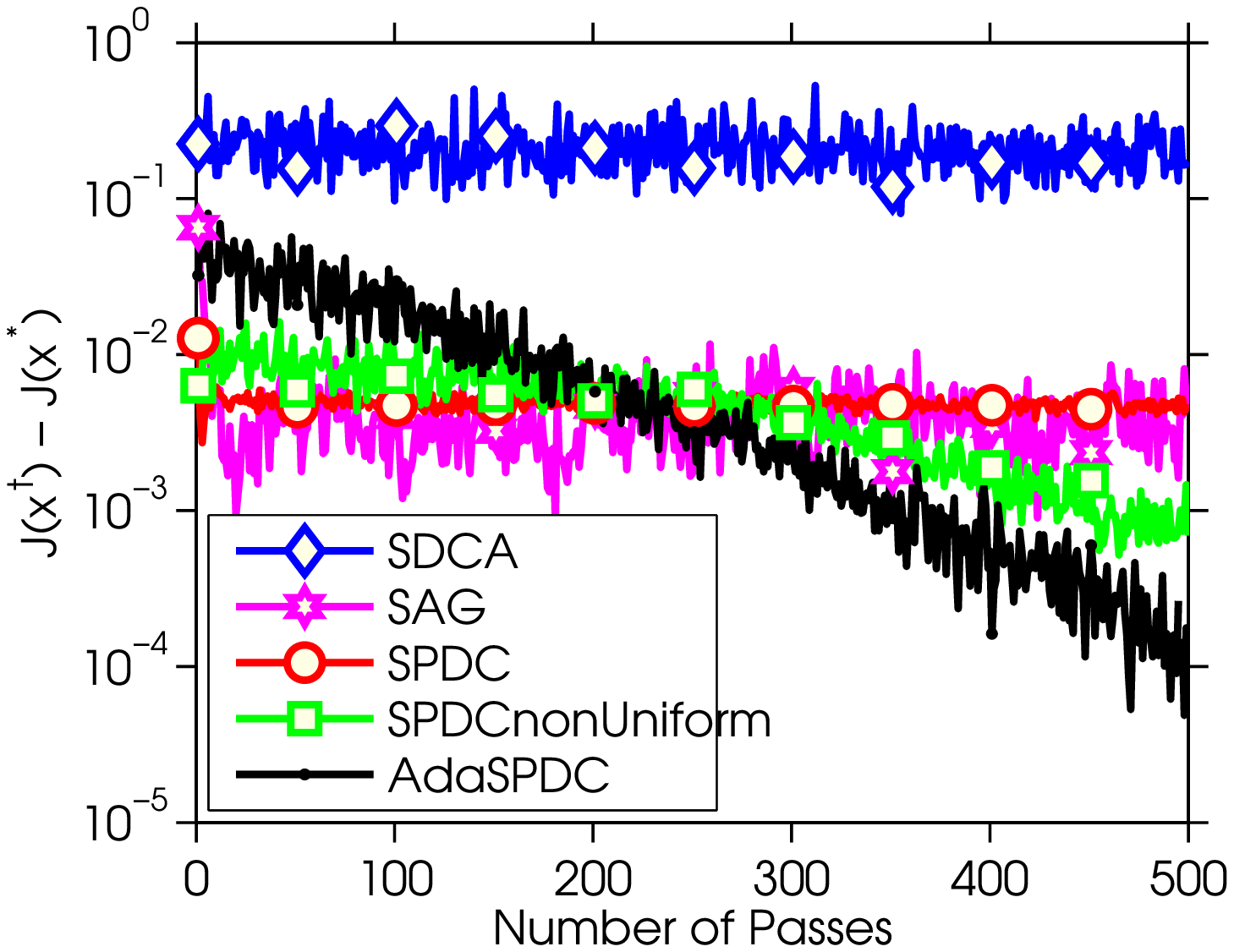}}\\
\hline
protein & \raisebox{-.5\totalheight}{\includegraphics[width=0.3\textwidth]{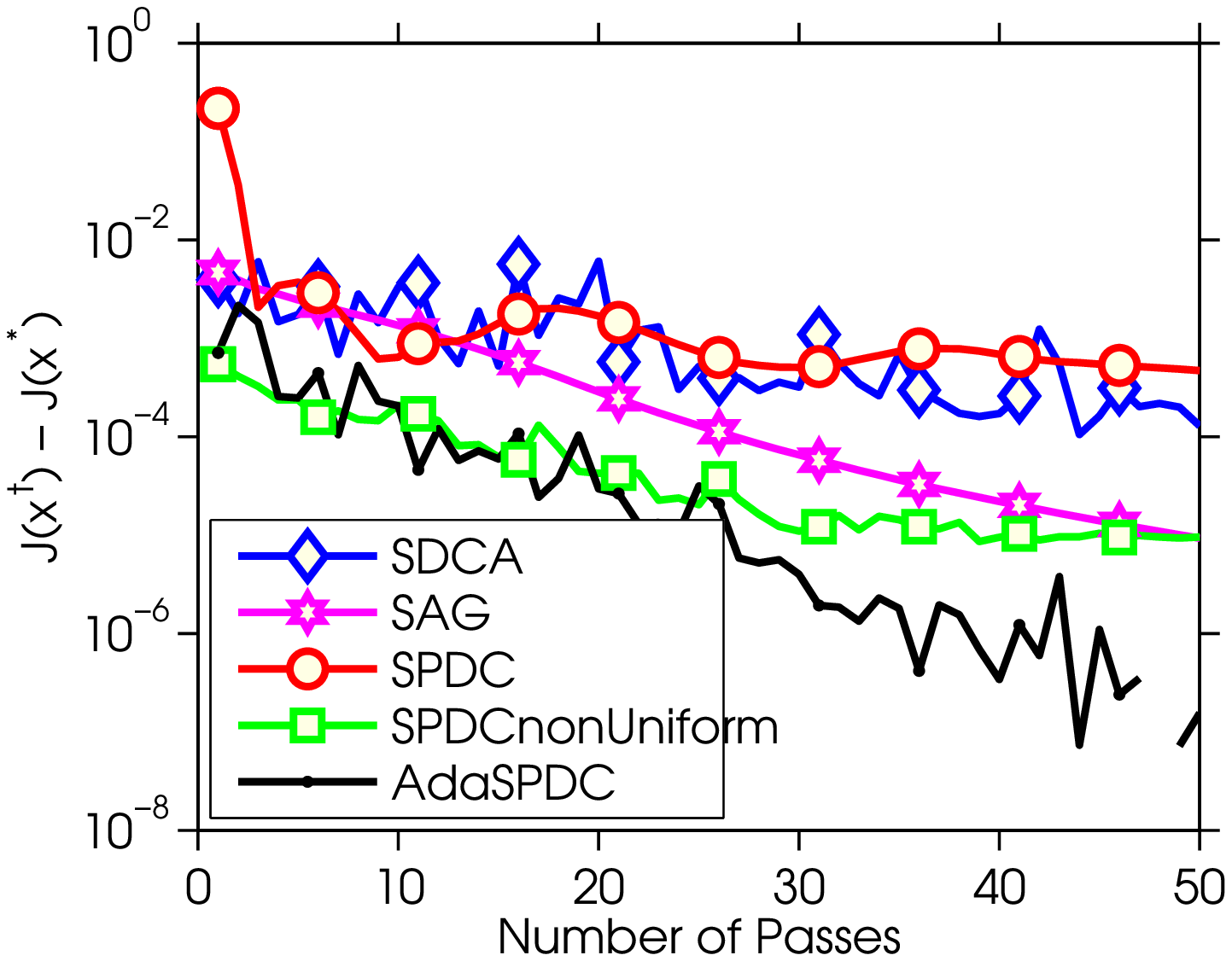}}&\raisebox{-.5\totalheight}{\includegraphics[width=0.3\textwidth]{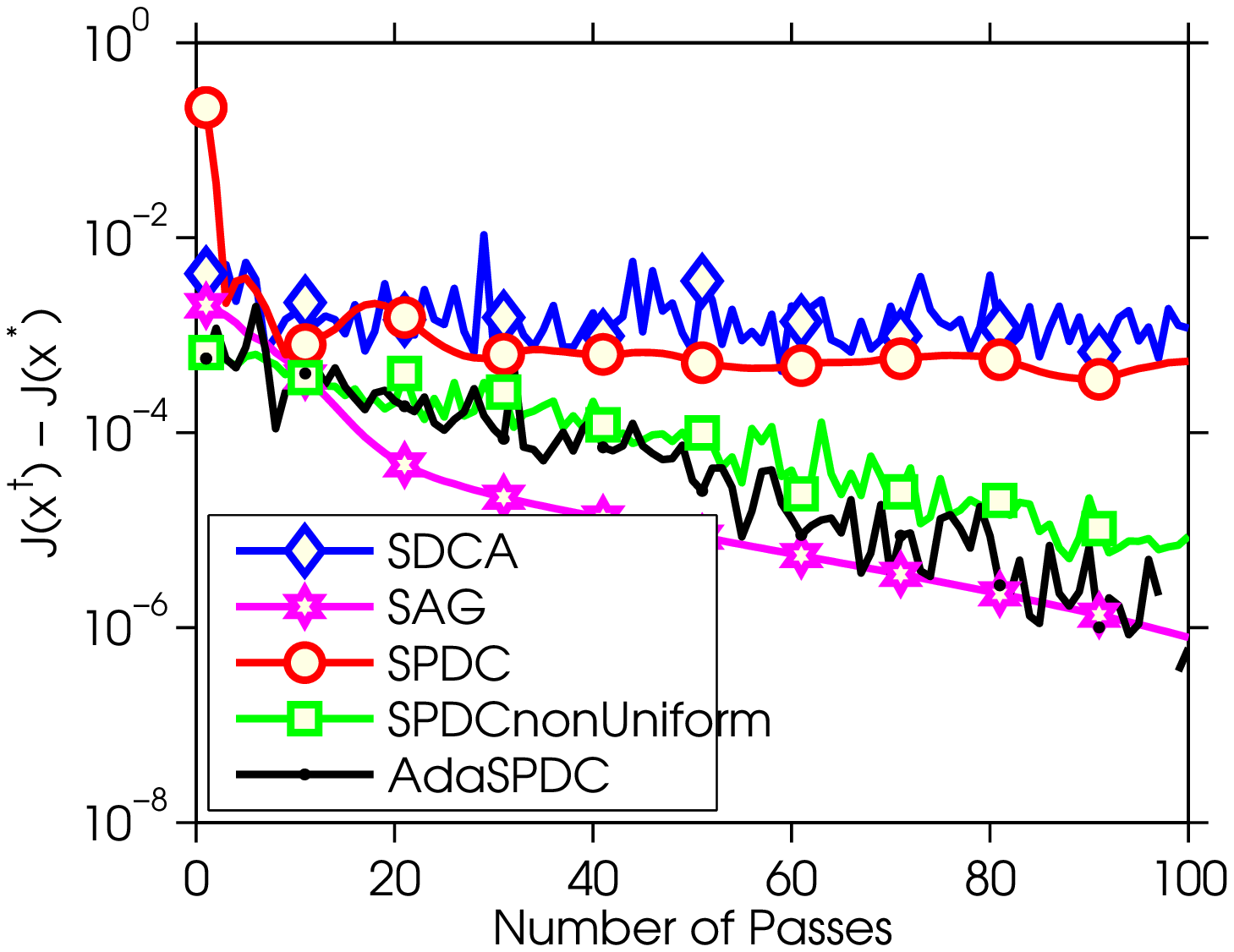}} & \raisebox{-.5\totalheight}{\includegraphics[width=0.3\textwidth]{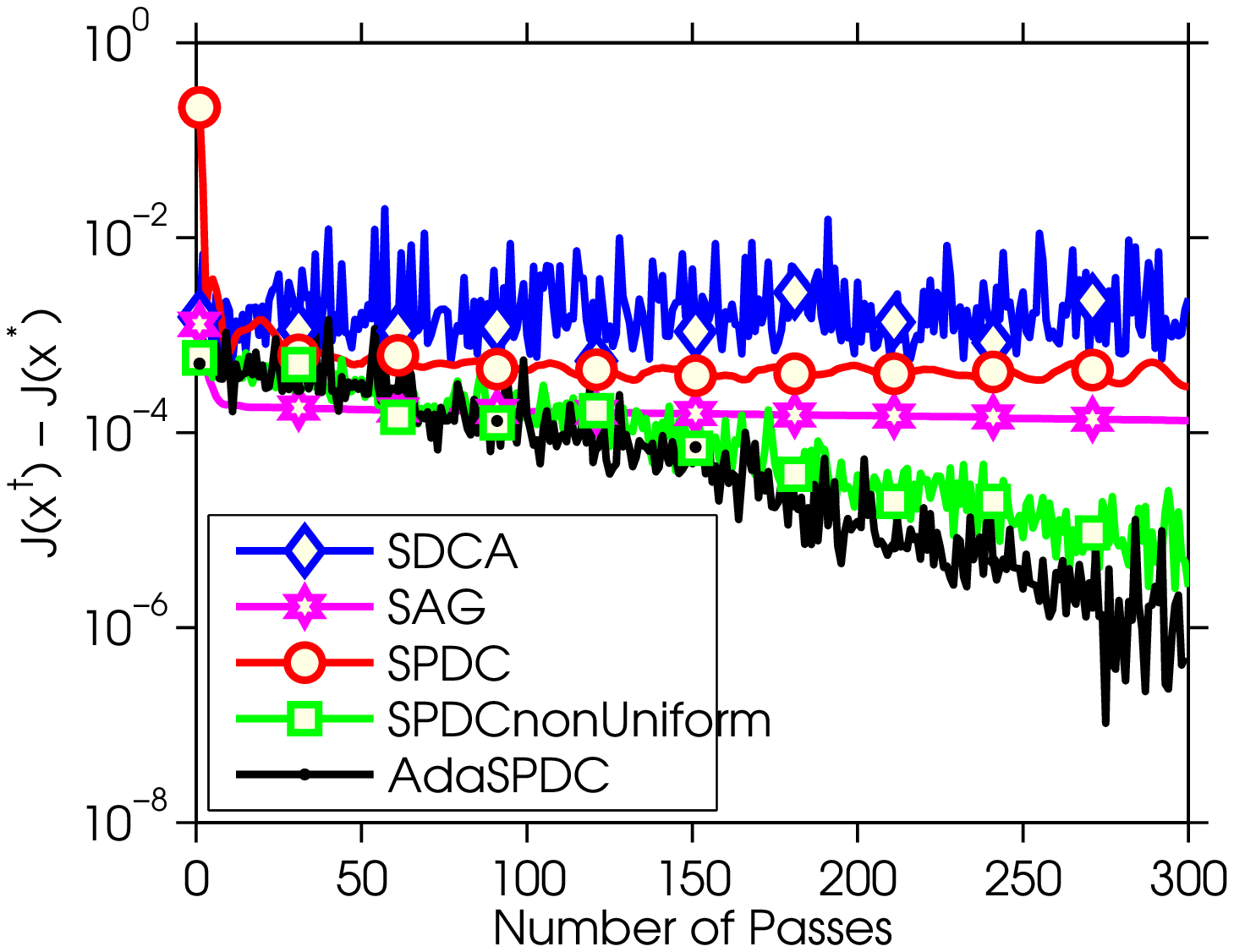}}
\end{tabular}
\vspace{-3mm}
\end{center}
}
\caption{\small Comparison of algorithm performance with smooth Hinge loss.
\label{fig:smoothhinge}}
\vspace{-3mm}
\end{figure*}

\begin{figure*}[ht!]
\vskip -0.1in
{\tiny
\begin{center}
\begin{tabular}{c | ccc}
  Dataset & $\lambda=10^{-5}$ & $\lambda=10^{-6}$ & $\lambda=10^{-7}$\\
\hline
w8a & \raisebox{-.5\totalheight}{\includegraphics[width=0.3\textwidth]{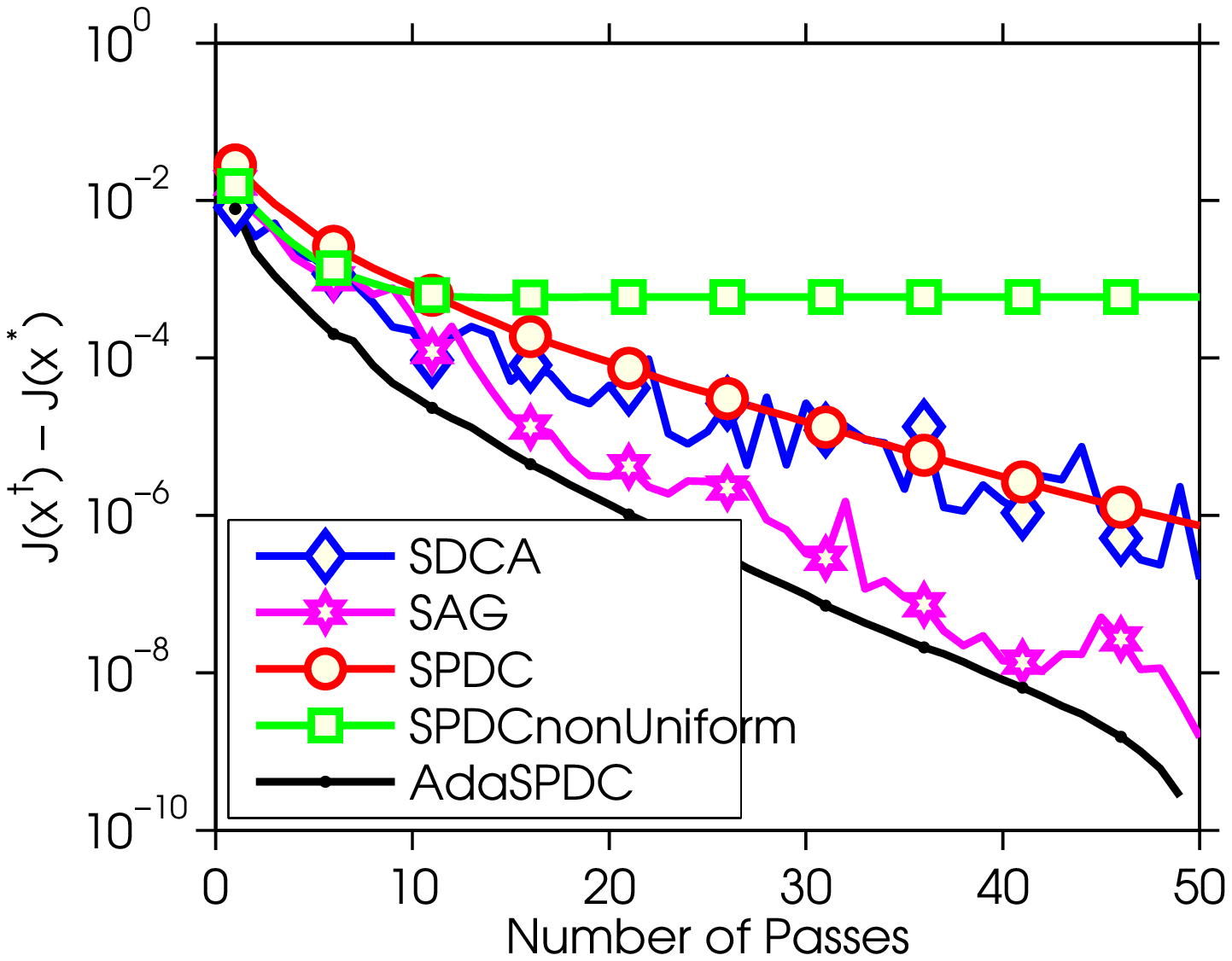}} & \raisebox{-.5\totalheight}{\includegraphics[width=0.3\textwidth]{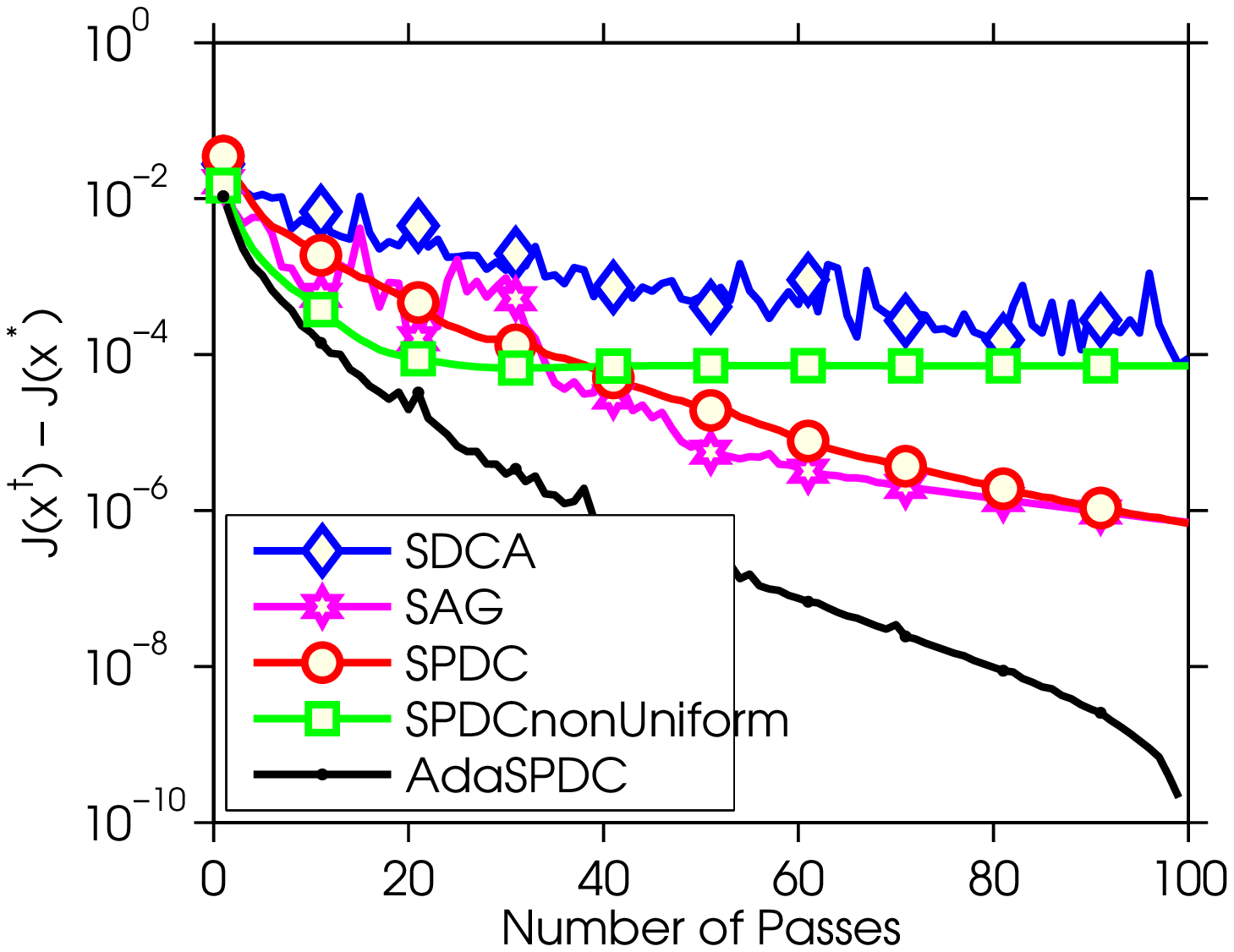}} & \raisebox{-.5\totalheight}{\includegraphics[width=0.3\textwidth]{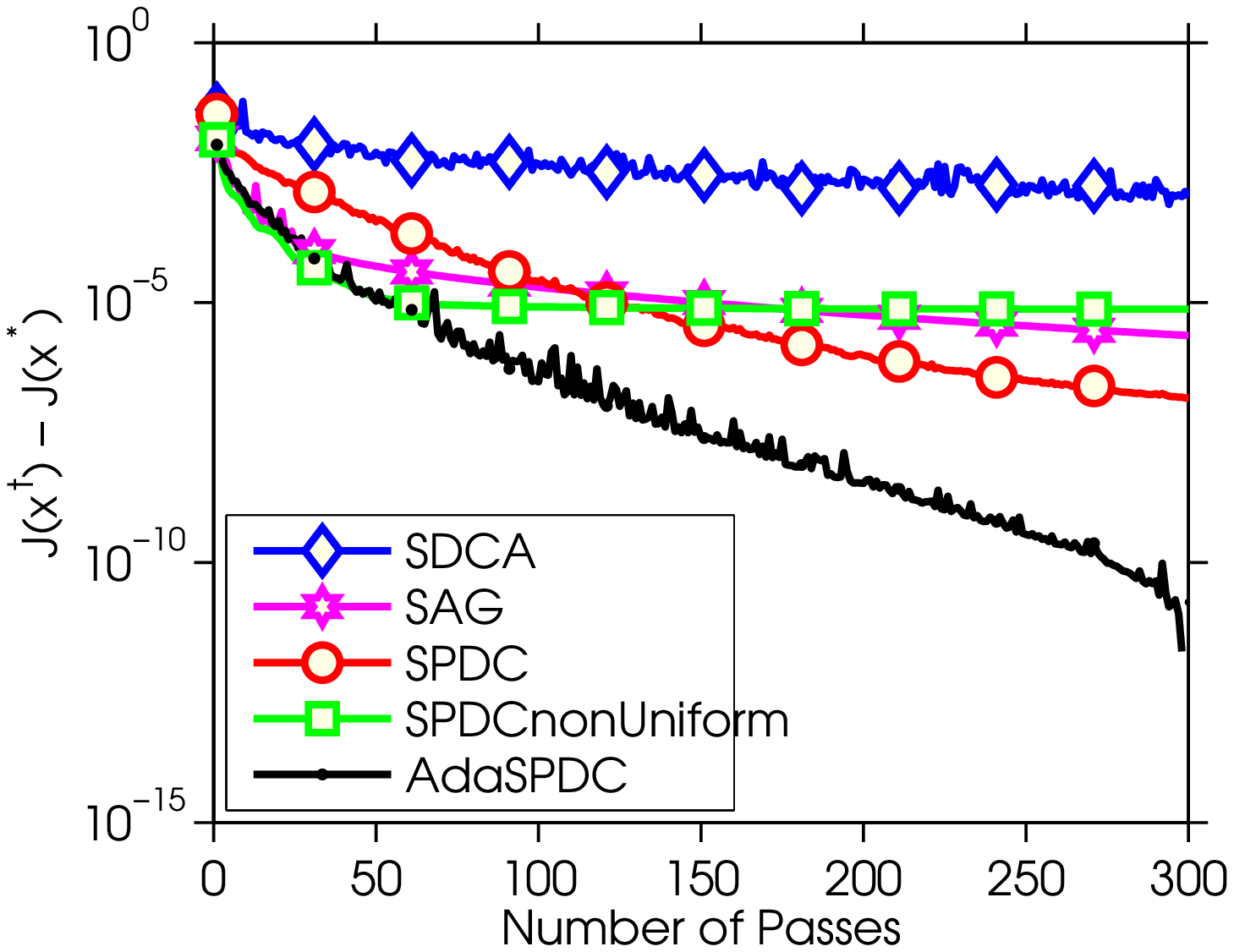}}\\
\hline
covertype & \raisebox{-.5\totalheight}{\includegraphics[width=0.3\textwidth]{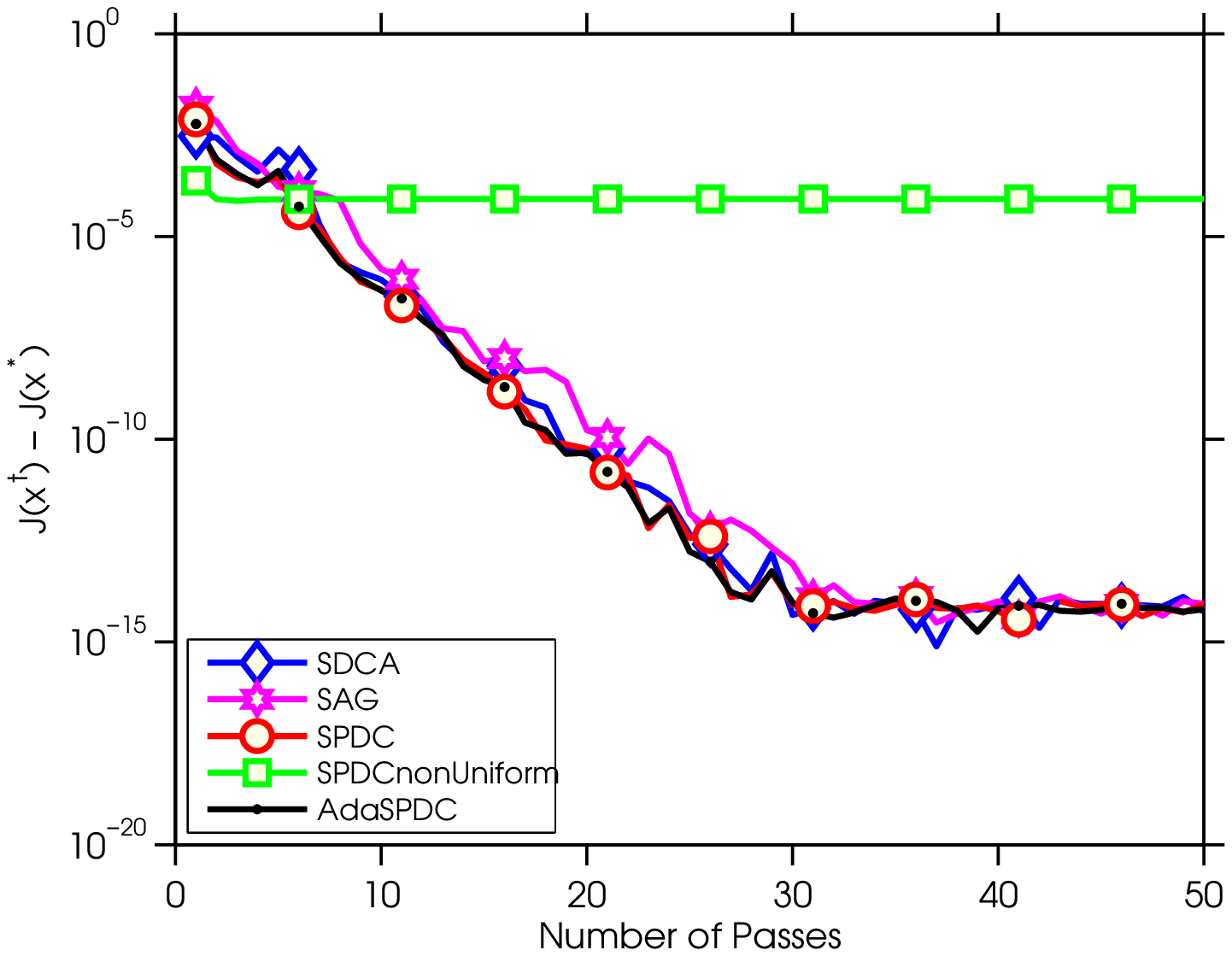}} & \raisebox{-.5\totalheight}{\includegraphics[width=0.3\textwidth]{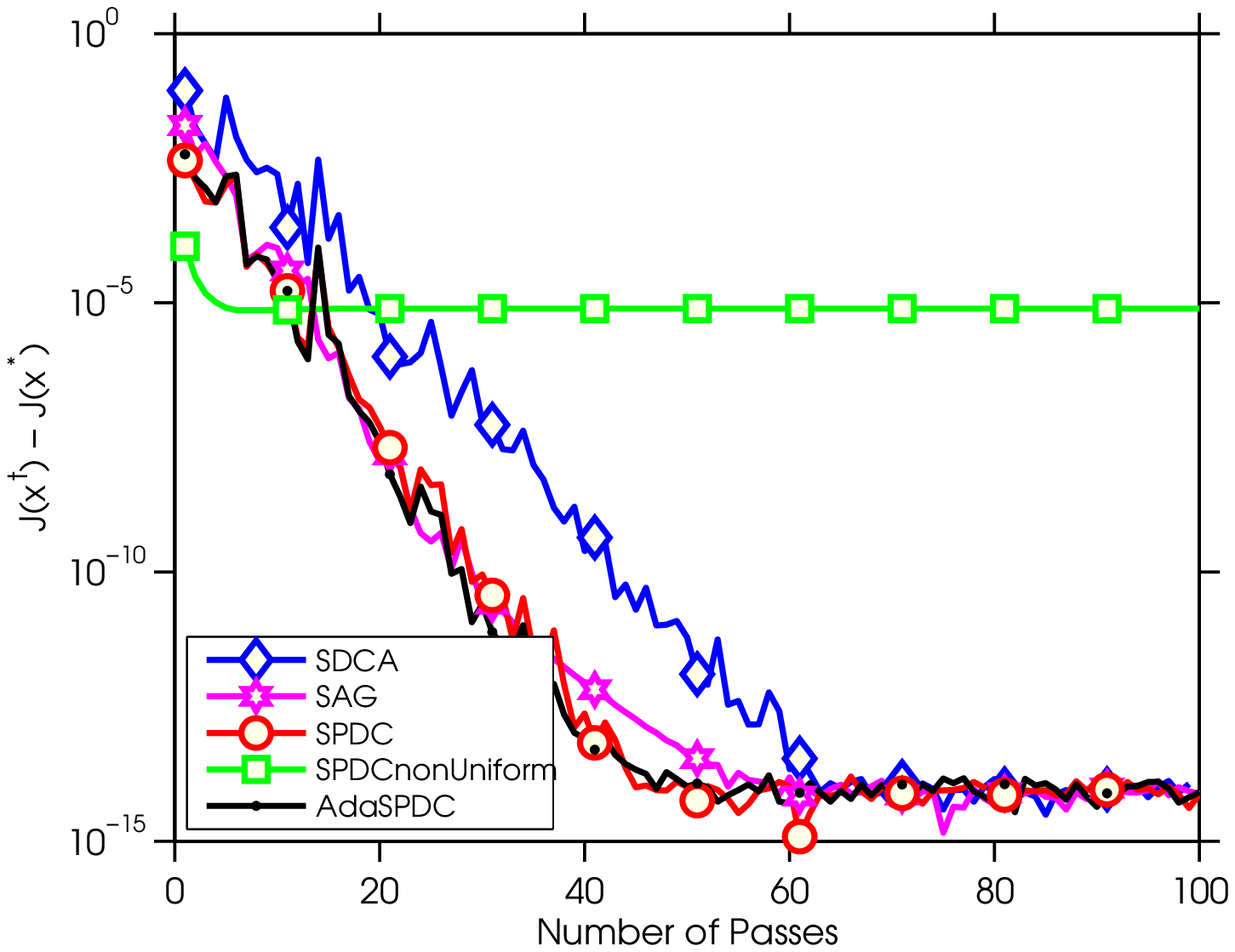}} & \raisebox{-.5\totalheight}{\includegraphics[width=0.3\textwidth]{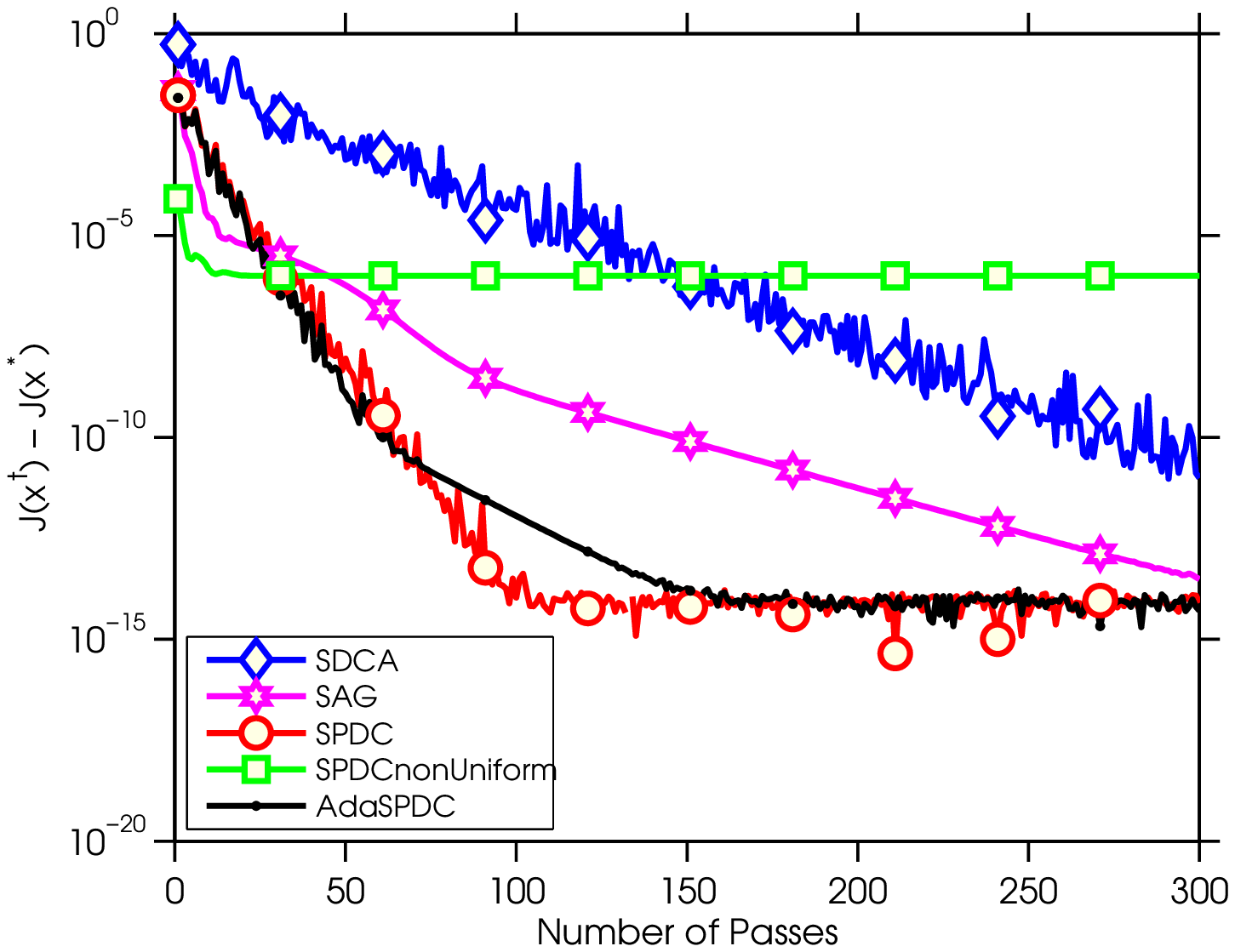}}\\
\hline
url & \raisebox{-.5\totalheight}{\includegraphics[width=0.3\textwidth]{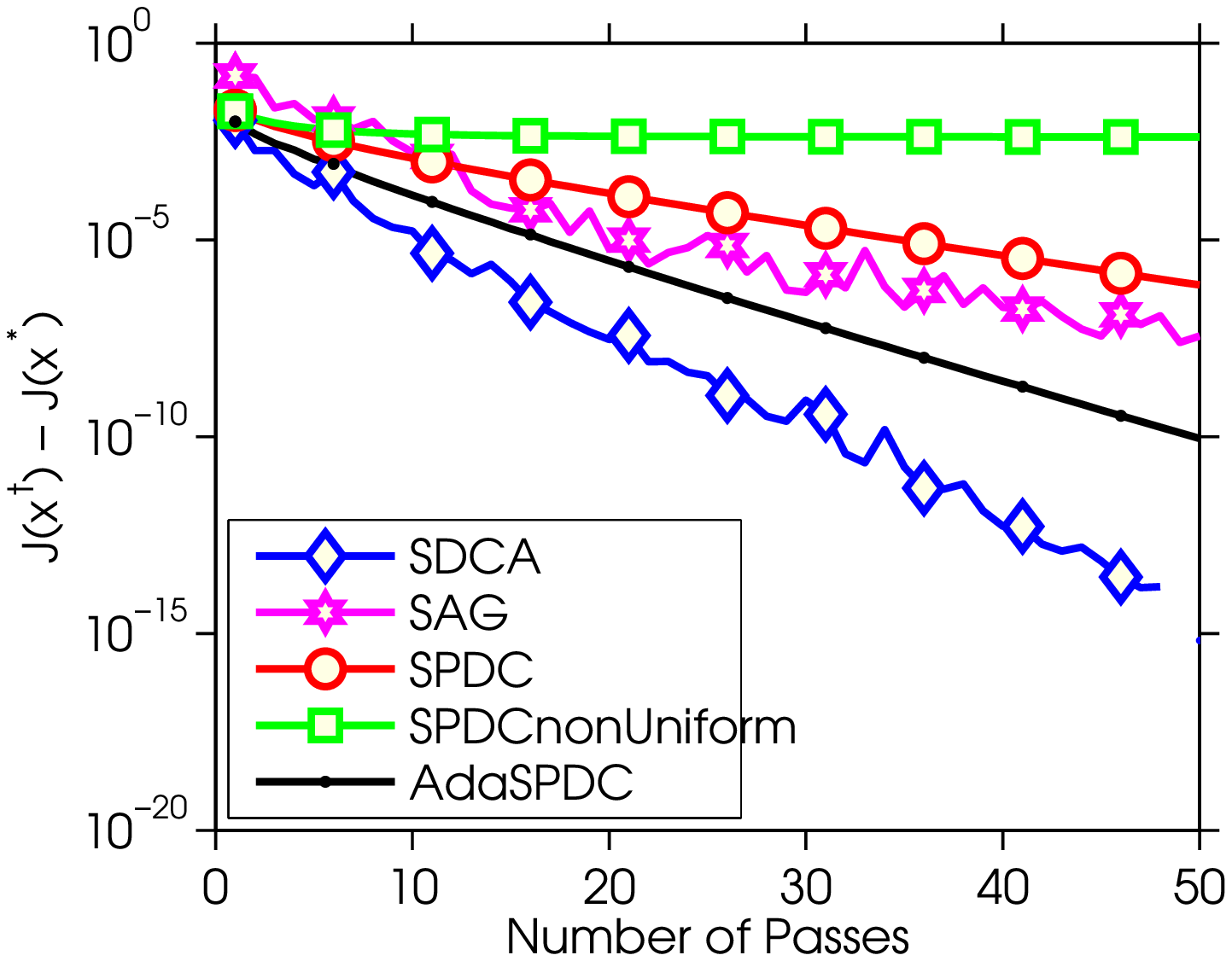}} & \raisebox{-.5\totalheight}{\includegraphics[width=0.3\textwidth]{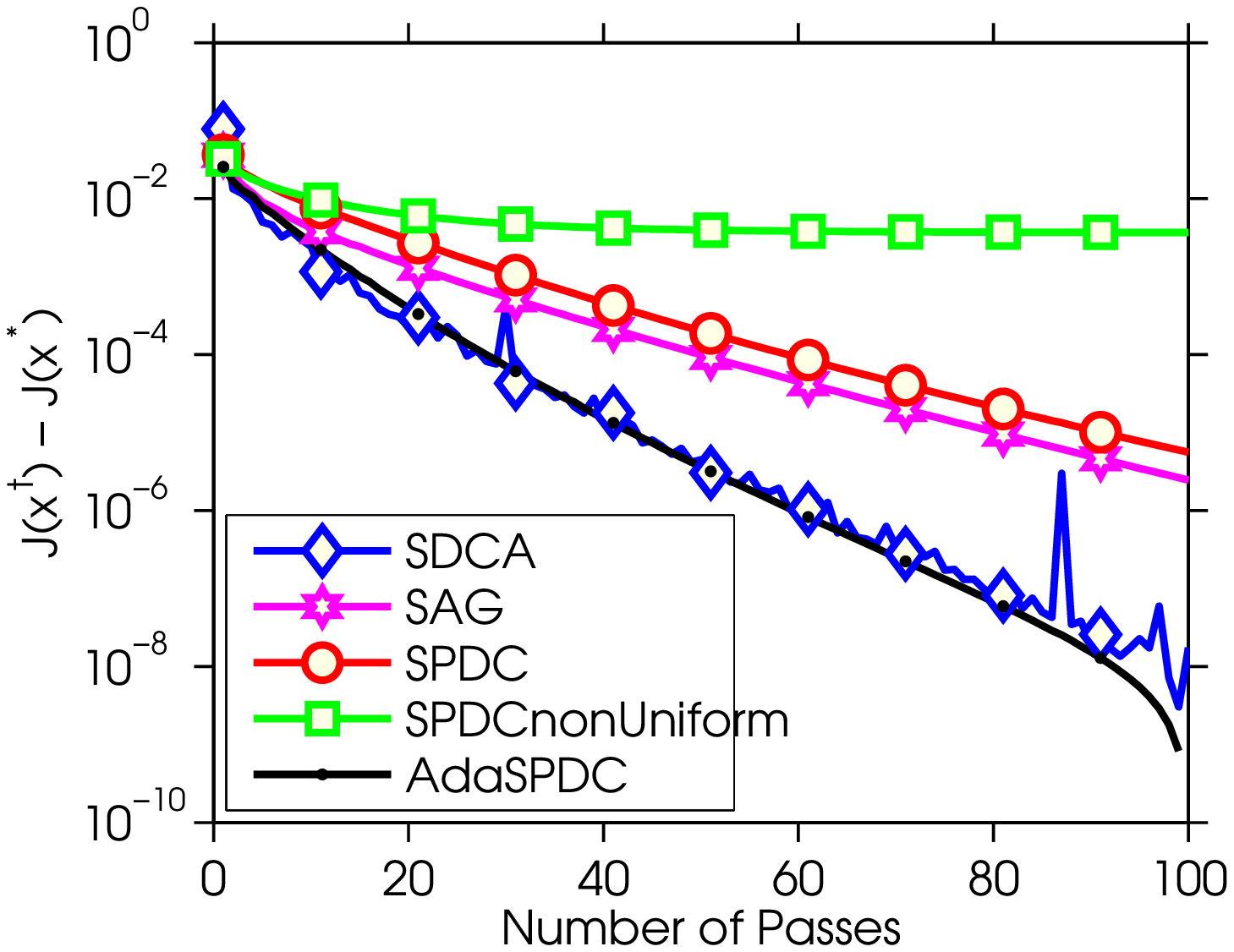}} & \raisebox{-.5\totalheight}{\includegraphics[width=0.3\textwidth]{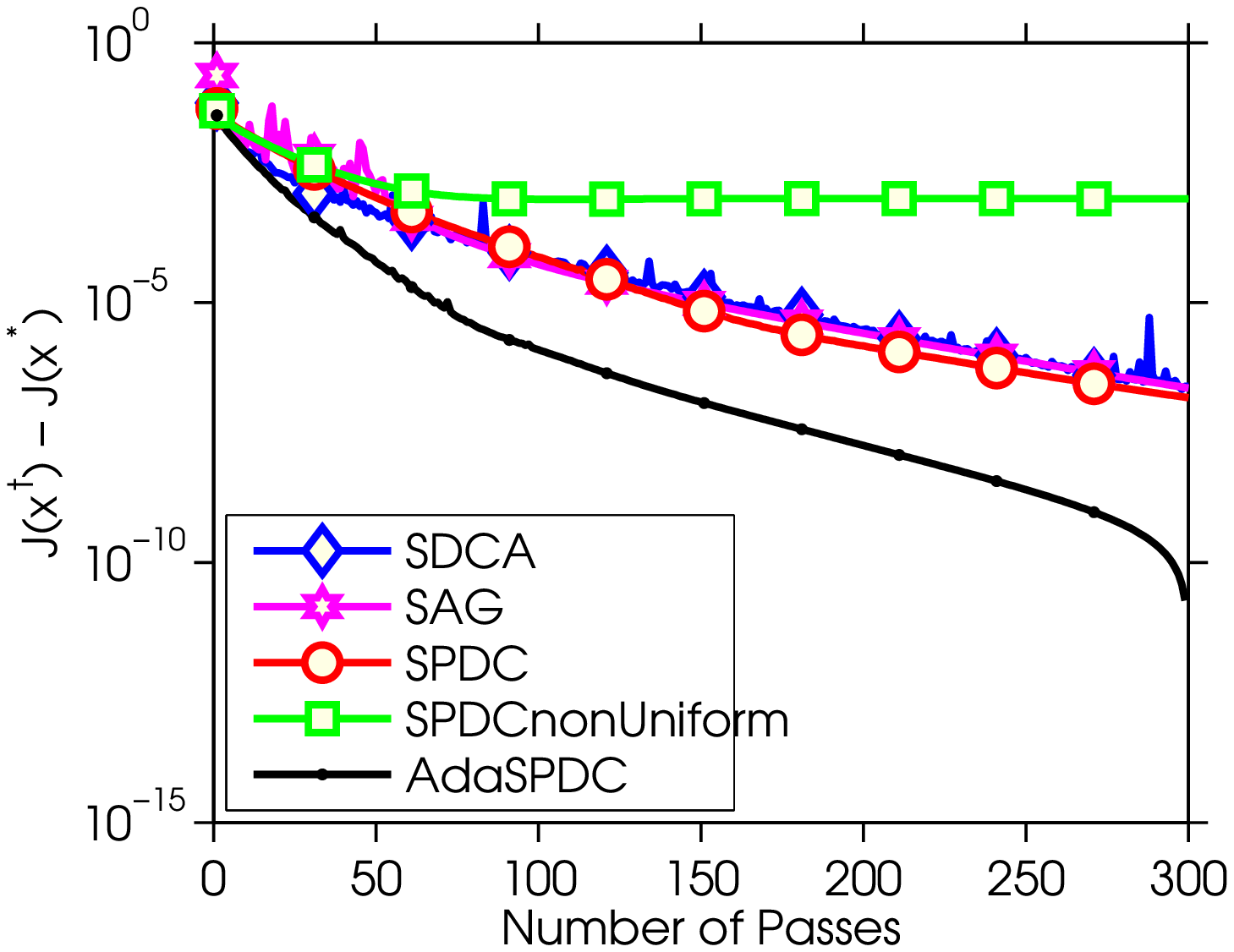}}\\
\hline
quantum & \raisebox{-.5\totalheight}{\includegraphics[width=0.3\textwidth]{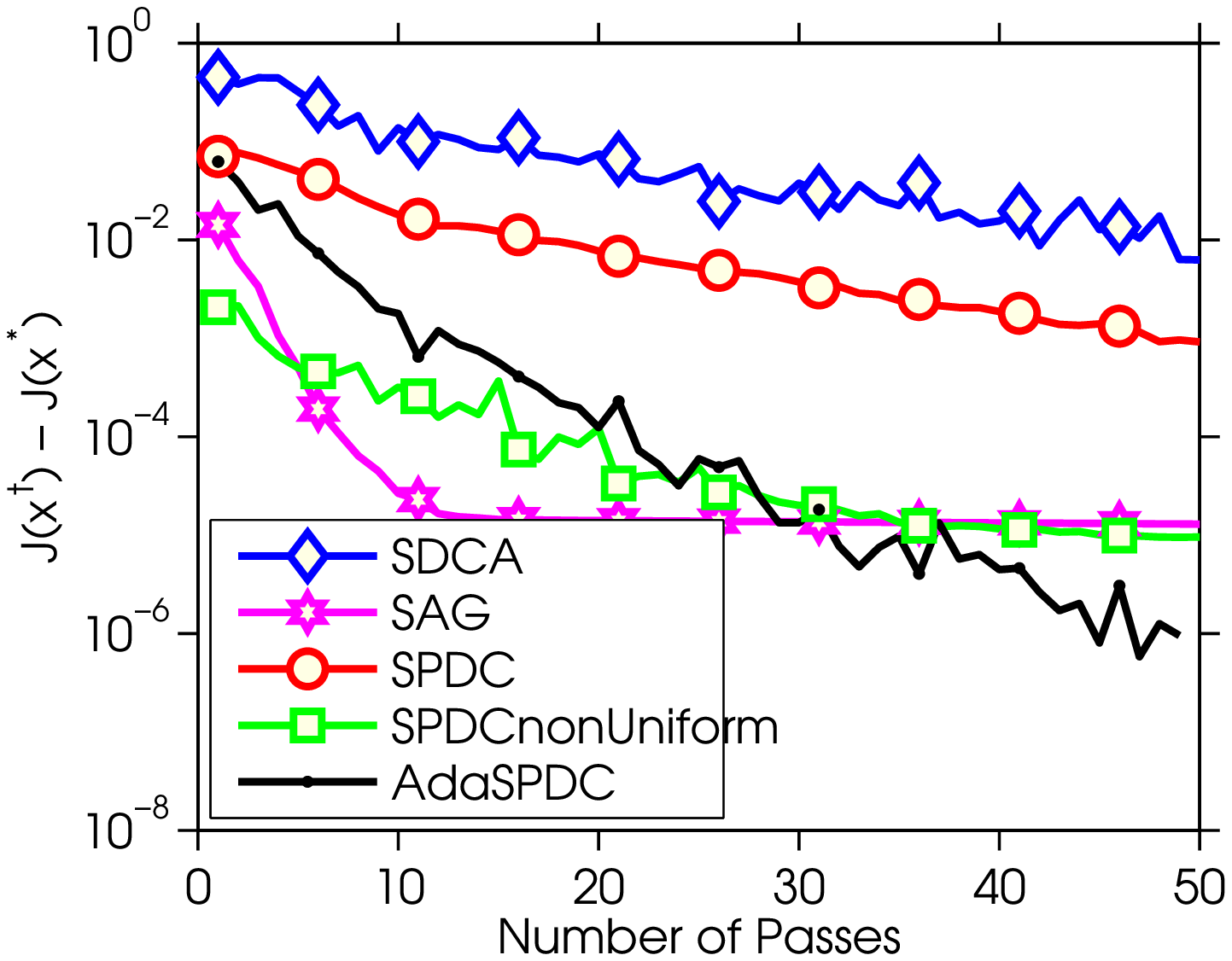}}&\raisebox{-.5\totalheight}{\includegraphics[width=0.3\textwidth]{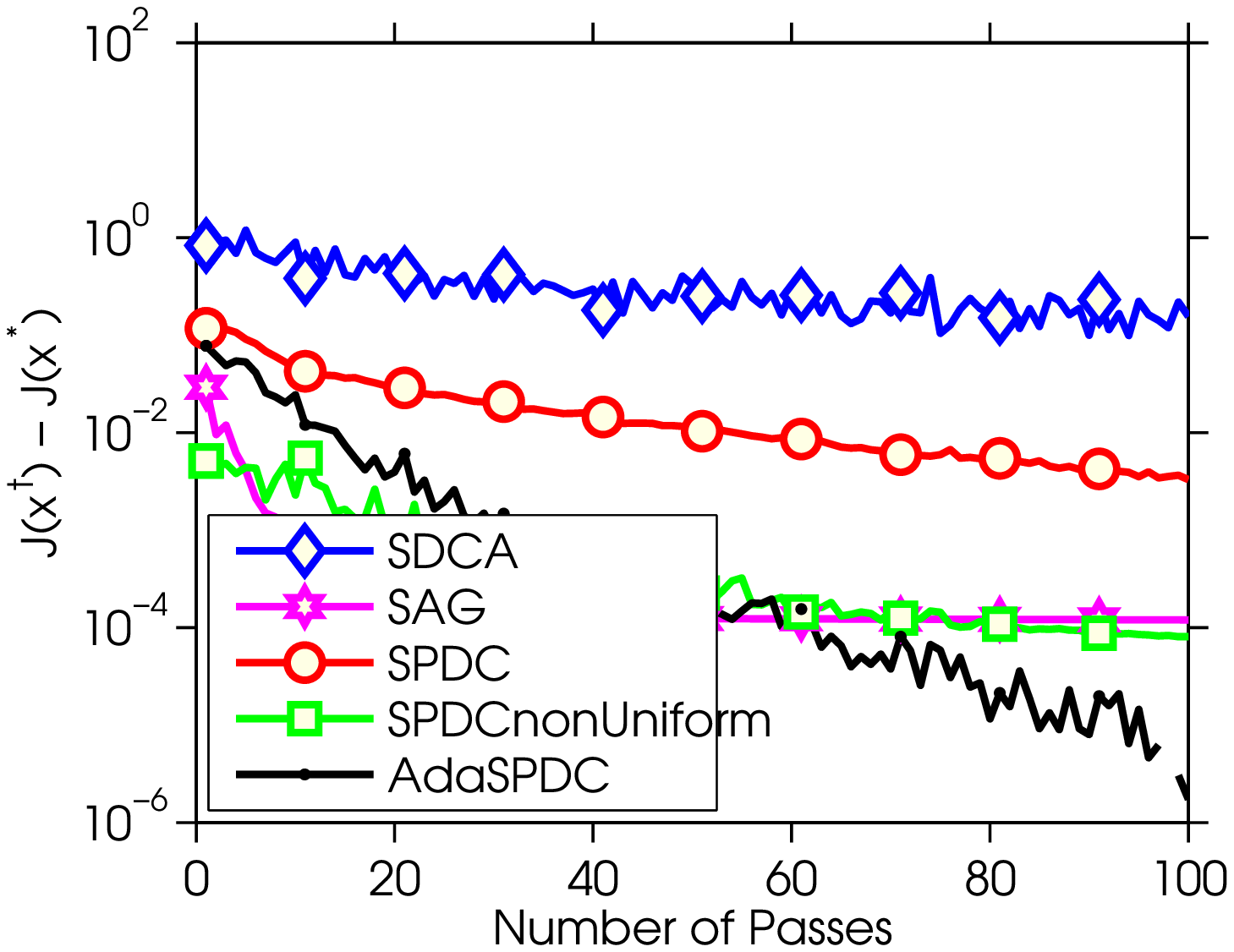}} & \raisebox{-.5\totalheight}{\includegraphics[width=0.3\textwidth]{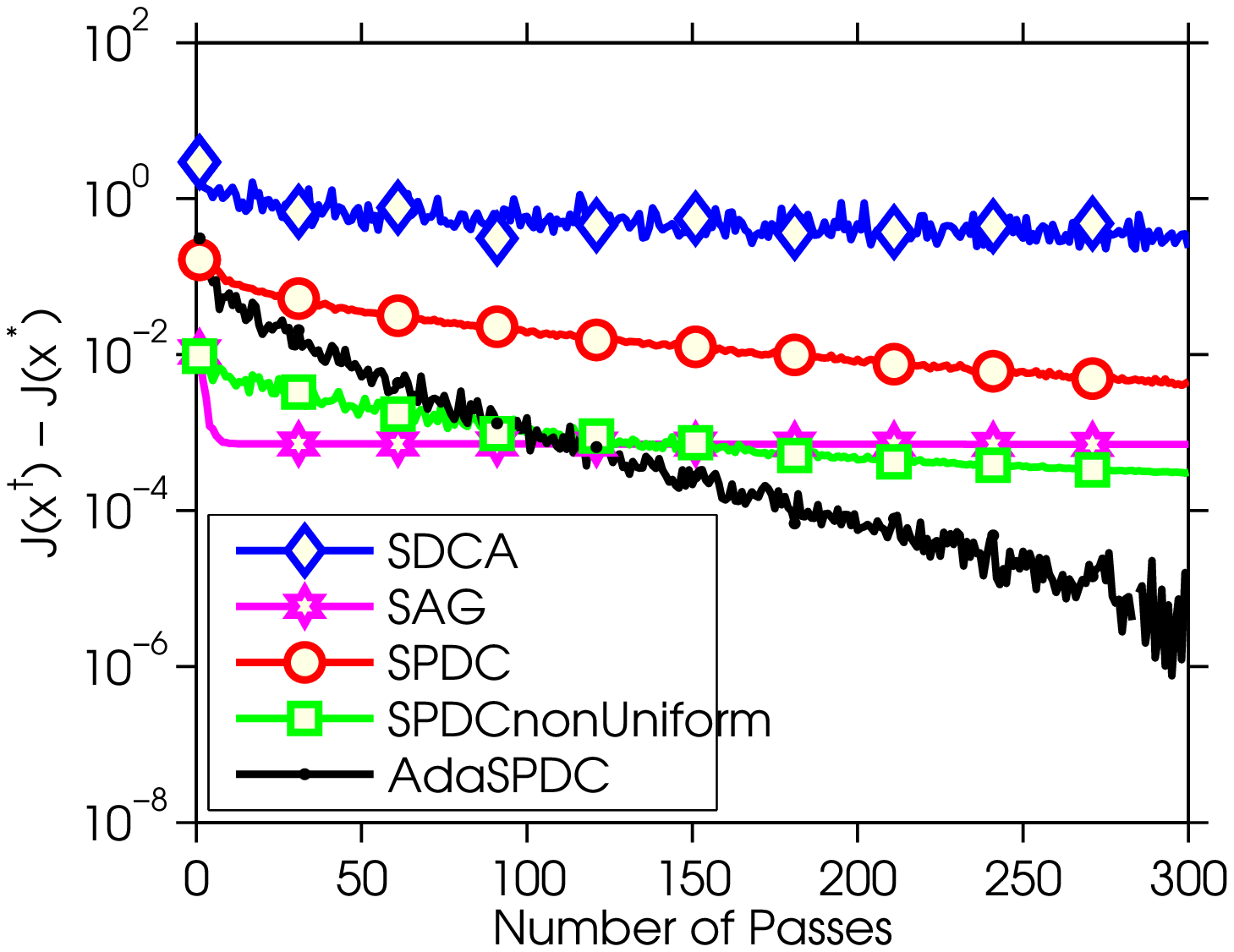}}\\
\hline
protein & \raisebox{-.5\totalheight}{\includegraphics[width=0.3\textwidth]{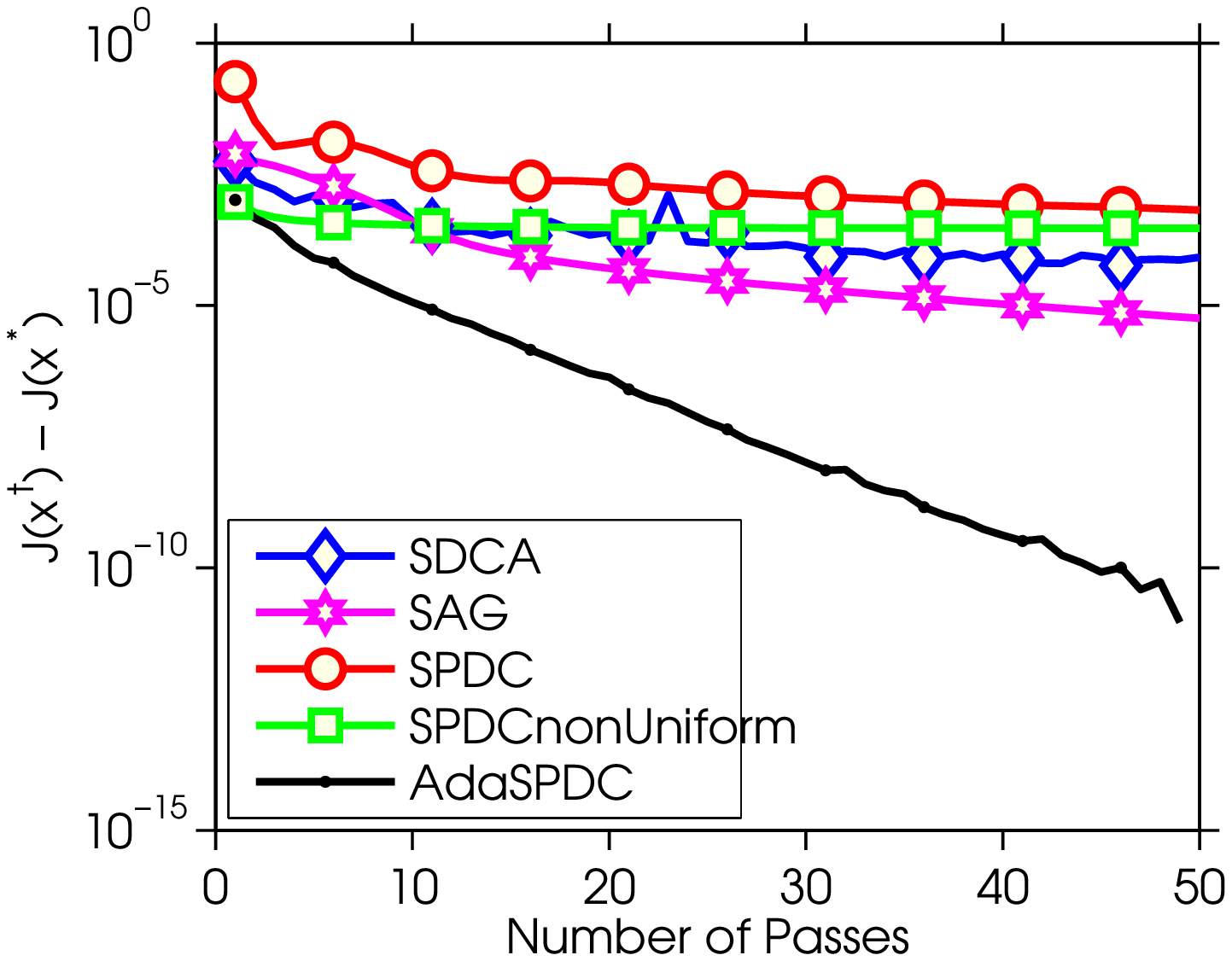}}&\raisebox{-.5\totalheight}{\includegraphics[width=0.3\textwidth]{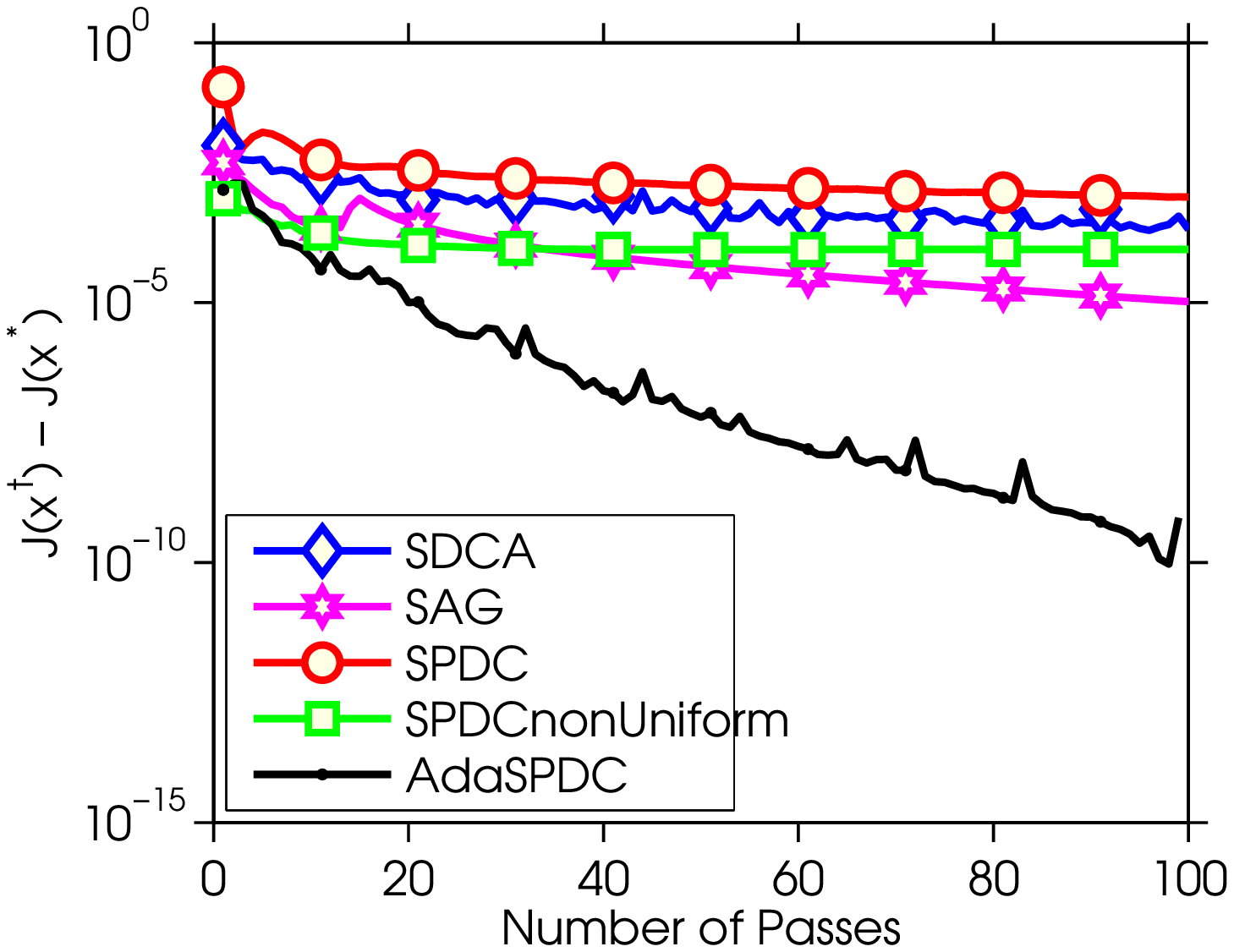}} & \raisebox{-.5\totalheight}{\includegraphics[width=0.3\textwidth]{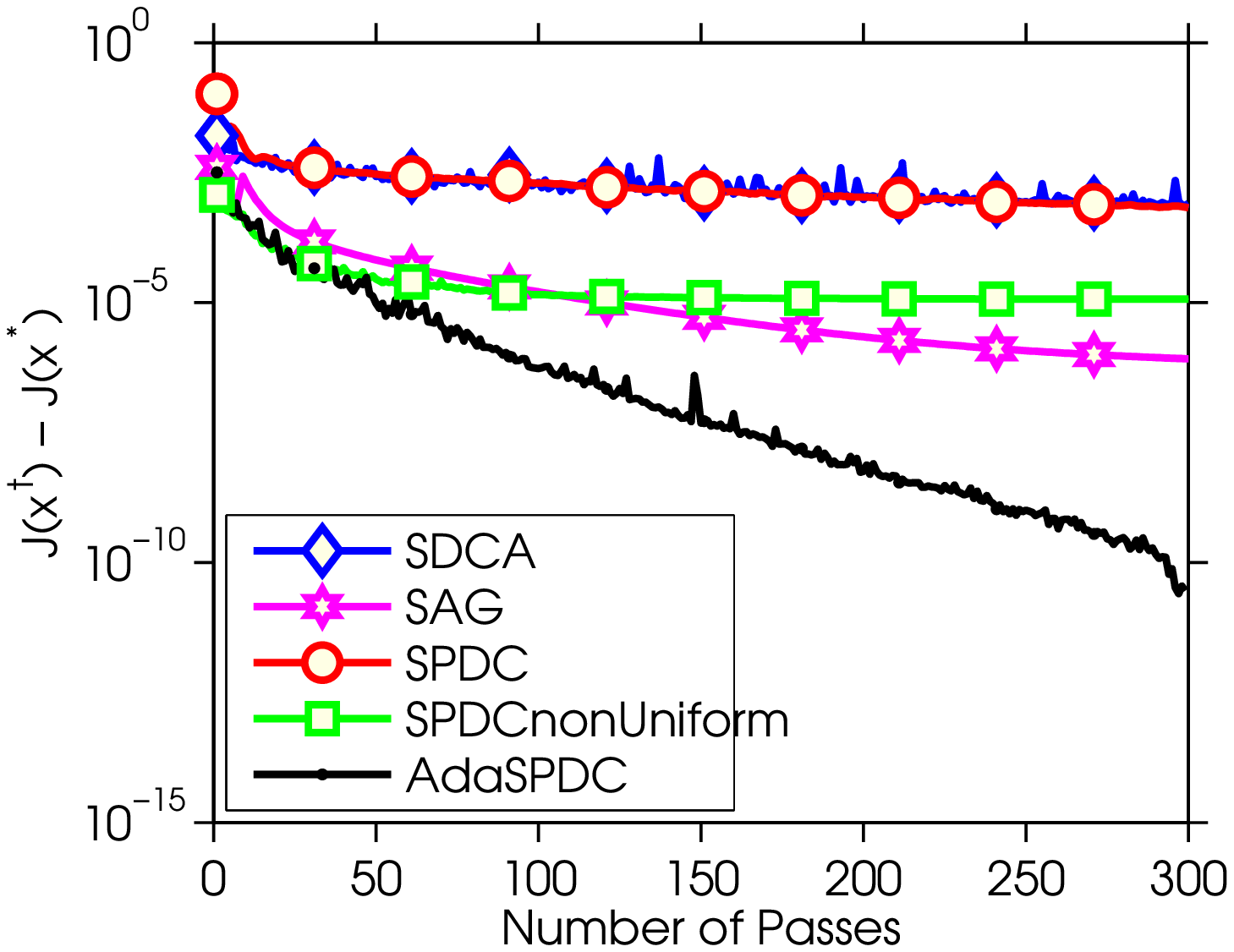}}
\end{tabular}
\vspace{-3mm}
\end{center}
}
\caption{\small Comparison of algorithm performance with Logistic loss.
\label{fig:logistic}}
\vspace{-3mm}
\end{figure*}

\subsection{Binary Classification on Real-world Datasets}
We now compare the performance of our  method AdaSPDC with other competitive methods on several real-world data sets. Our experiments focus on the freely-available benchmark data sets for binary classification, whose detailed information are listed in Table~\ref{tab:datasets}. The \emph{w8a}, \emph{covertype} and \emph{url} data are obtained from the LIBSVM collection\footnote{\url{http://www.csie.ntu.edu.tw/~cjlin/libsvmtools/datasets/binary.html}}. The \emph{quantum} and \emph{protein} data sets are obtained from KDD Cup 2004\footnote{\url{http://osmot.cs.cornell.edu/kddcup/datasets.html}}.  For all the datasets, each sample takes the form $(\aB_i, b_i)$ with $\aB_i$ is the feature vector and $b_i$ is the binary label $-1$ or $1$. We add a bias term to the feature vector for all the datasets.  We aim to minimize the regularized empirical risk with following form
\begin{equation}
J(\xB) = \frac{1}{n} \sum_{i=1}^n \phi_i(\aB_i^T \xB) + \frac{\lambda}{2} \| \xB\|_2^2
\end{equation}
To provide a more comprehensive comparison between these methods, we experiment with two different loss function $\phi_i(\cdot)$, smooth Hinge loss \cite{shalev2013} and logistic loss, described in the following.
\paragraph{Smooth Hinge loss} (with smoothing parameter $\gamma = 1$.)
\begin{equation*}
\phi_i(z) =
\begin{cases}
0 & \text{if } b_i z \geq 1,\\
1-\frac{\gamma}{2} - b_i z & \text{if } b_i z  \leq 1-\gamma\\
\frac{1}{2\gamma} (1-b_i z)^2 & \text{otherwise}.
\end{cases}
\end{equation*}
And its conjugate dual is
\begin{equation*}
\phi_i^*(y_i) = b_i y_i + \frac{1}{2} y_i^2, \text{ with }b_i y_i \in [-1,0 ].
\end{equation*}
We can observe that $\phi_i^*(y_i)$ is $\gamma$-strongly convex with $\gamma=1$. The dual update of AdaSPDC for smooth Hinge loss is nearly the same with ridge regression except the necessity of projection into the interval $b_i y_i \in [-1,0 ]$.
\paragraph{Logistic loss}
\begin{equation*}
\phi_i(z) = \log \left( 1 + \exp (-b_i z) \right),
\end{equation*}
whose conjugate dual has the form
\begin{equation*}
\phi_i^*(y_i) = -b_i y_i \log (-b_i y_i) + (1+ b_i y_i) \log (1+b_i y_i) \text{ with }b_i y_i \in [-1,0 ].
\end{equation*}
It is also easy to obtain that $\phi_i^*(y_i)$ is $\gamma$-strongly convex with $\gamma = 4$.
Note that for logistic loss, the dual update in Eq.~(\ref{eq:dualupdate}) does not have a closed form solution, and we can start from some initial solution and further apply several steps of Newton's update to obtain a more accurate solution.

During the experiments, we observe that the performance of SAG is very sensitive to the stepsize choice. To obtain best results of SAG, we try different choices of stepsize in the interval $[1/16L, 1/L]$ and report the best result for each dataset, where $L$ is Lipschitz constant of $\phi_i(\aB_i^T \xB)$, $1/16L$ is the theoretical stepsize choice for SAG and $1/L$ is the suggested empirical choice \cite{schmidt2013}. For smooth Hinge loss, $L = \max_i \{ \| \aB_i\|_2, i = 1, \dots, n \}$, and for logistic loss, $L = \frac{1}{4}\max_i \{ \| \aB_i\|_2, i = 1, \dots, n \}$.

Fig.~\ref{fig:smoothhinge} and Fig.~\ref{fig:logistic} depict the algorithm performance on the different methods with smooth Hinge loss and logistics loss, respectively. We compare all these methods with different values of $\lambda = \{10^{-5}, 10^{-6}, 10^{-7} \}$. Generally, our method AdaSPDC performs consistently better or at least comparably with other methods, and performs especially well for the tasks with small regularized parameter $\lambda$. For some datasets, such as covertype and quantum, SPDC with non-uniform sampling decreases the objective faster than other methods in early epochs, however, cannot achieve comparable results with other methods in later epochs, which might be caused by its conservative stepsize.

\section{Conclusion \&  Future Work}
\label{sec:con}

In this work, we propose Adaptive Stochastic Primal-Dual Coordinate Descent (AdaSPDC) for separable saddle point problems. As a non-trivial extension of a recent work SPDC \cite{zhang2014}, AdaSPDC uses an adaptive step size choices for both primal and dual updates in each iteration. The design of the step size for our method AdaSPDC explicitly and adaptively models the coupling strength between chosen block coordinates and primal variable through the spectral norm of each $\AB_i$. We theoretically characterise that AdaSPDC holds a sharper linear convergence rate than SDPC. Additionally, we demonstrate the superiority of the proposed AdaSPDC method on ERM problems through extensive experiments on both synthetic and real-world data sets.

An immediate further research direction is to investigate other valid parameter configurations for the extrapolation parameter $\theta$, and the primal and dual step sizes $\tau$ and $\sigma$ both theoretically and empirically. In addition, discovering the potential theoretical connections with other stochastic optimization methods will also be enlightening.

\subsubsection*{Acknowledgments.} Z. Zhu is supported by China Scholarship Council/University of Edinburgh Joint Scholarship. The authors would like to thank Jinli Hu for insightful discussion on the proof of Theorem 1.

\bibliography{AdaSPDC_full_version}

\begin{thebibliography}{10}
\providecommand{\url}[1]{\texttt{#1}}
\providecommand{\urlprefix}{URL }

\bibitem{chambolle2011}
Chambolle, A., Pock, T.: A first-order primal-dual algorithm for convex
  problems with applications to imaging. Journal of Mathematical Imaging and
  Vision  40(1),  120--145 (2011)

\bibitem{esser2010}
Esser, E., Zhang, X., Chan, T.: A general framework for a class of first order
  primal-dual algorithms for convex optimization in imaging science. SIAM
  Journal on Imaging Sciences  3(4),  1015--1046 (2010)

\bibitem{hastie2009}
Hastie, T., Tibshirani, R., Friedman, J.: The elements of statistical learning,
  vol.~2. Springer (2009)

\bibitem{he2012}
He, B., Yuan, X.: Convergence analysis of primal-dual algorithms for a
  saddle-point problem: from contraction perspective. SIAM Journal on Imaging
  Sciences  5(1),  119--149 (2012)

\bibitem{he2014}
He, Y., Monteiro, R.D.: An accelerated hpe-type algorithm for a class of
  composite convex-concave saddle-point problems. Optimization-online preprint
  (2014)

\bibitem{jacob2009}
Jacob, L., Obozinski, G., Vert, J.P.: Group lasso with overlap and graph lasso.
  In: Proceedings of the 26th annual international conference on machine
  learning. pp. 433--440. ACM (2009)

\bibitem{nesterov2004}
Nesterov, Y.: Introductory lectures on convex optimization: A basic course,
  vol.~87. Springer (2004)

\bibitem{nesterov2012efficiency}
Nesterov, Y.: Efficiency of coordinate descent methods on huge-scale
  optimization problems. SIAM Journal on Optimization  22(2),  341--362 (2012)

\bibitem{ouyang2014}
Ouyang, Y., Chen, Y., Lan, G., Pasiliao~Jr, E.: An accelerated linearized
  alternating direction method of multipliers. arXiv preprint arXiv:1401.6607
  (2014)

\bibitem{richtarik2012parallel}
Richt{\'a}rik, P., Tak{\'a}{\v{c}}, M.: Parallel coordinate descent methods for
  big data optimization. arXiv preprint arXiv:1212.0873  (2012)

\bibitem{richtarik2014iteration}
Richt{\'a}rik, P., Tak{\'a}{\v{c}}, M.: Iteration complexity of randomized
  block-coordinate descent methods for minimizing a composite function.
  Mathematical Programming  144(1-2),  1--38 (2014)

\bibitem{schmidt2013}
Schmidt, M., Roux, N.L., Bach, F.: Minimizing finite sums with the stochastic
  average gradient. arXiv preprint arXiv:1309.2388  (2013)

\bibitem{shalev2013}
Shalev-Shwartz, S., Zhang, T.: Stochastic dual coordinate ascent methods for
  regularized loss. The Journal of Machine Learning Research  14(1),  567--599
  (2013)

\bibitem{tseng2008}
Tseng, P.: On accelerated proximal gradient methods for convex-concave
  optimization. submitted to SIAM Journal on Optimization  (2008)

\bibitem{zhang2014}
Zhang, Y., Xiao, L.: Stochastic primal-dual coordinate method for regularized
  empirical risk minimization. arXiv preprint arXiv:1409.3257  (2014)

\bibitem{zhu2008}
Zhu, M., Chan, T.: An efficient primal-dual hybrid gradient algorithm for total
  variation image restoration. UCLA CAM Report pp. 08--34 (2008)

\end{thebibliography}
\bibliographystyle{splncs03}

\section*{Appendix: Proofs}

Before presenting the proof of Theorem 1, we firstly provide the following lemma and its proof, which characterizes positive semi-definiteness of an important matrix used in the proof of Theorem 1.
\begin{lemma}
Given any matrix $\KB \in \Rbb^{d \times m}$, we partition the matrix $\KB$ into $J$ column blocks, $\KB_j \in \Rbb ^{d \times m_j}$, $j =1, \dots, J$, and then $\sum_{j=1}^J m_j = m$. We then define two diagonal matrices, $\UB = u \IB \in \Rbb^{d \times d}$, and $\VB = \text{diag} \left(v_1 \IB_{m_1}, v_2 \IB_{m_2}, \dots, v_J \IB_{m_J} \right) \in \Rbb^{m \times m}$ and let $\VB_j = v_j \IB_{m_j}$. And denote $R_j = \| \KB_j \|_2 = \sqrt{\mu_{\text{max}} \left( \KB_j^T \KB_j \right)}$, where $\| \cdot \|_2$ is the spectral norm and $\mu_{\max}(\cdot)$ is the maximum singular value of a matrix. And let $R_{\max} = \max \left\{R_j | j = 1,\dots, J \right\}$. Now we consider the following parameter configuration, for any positive constant $c>0$,
\begin{align}
v_j &= \frac{c}{R_j}, \quad j = 1, \dots, J \\
u  &= \frac{1}{c J R_{\max}}.
\end{align}
Under the above parameter configuration, the following matrix is positive definite,
\begin{equation}
\PB =
\begin{bmatrix}
\UB^{-1} & -\KB \\
-\KB^T & \VB^{-1}
\end{bmatrix}
\succ 0.
\end{equation}
\end{lemma}

\begin{proof}
Firstly consider each separable column block $\KB_j$, then
\begin{align}
\| \UB^{\frac{1}{2}} \KB_j \VB_j^{\frac{1}{2}} \|_2^2 \leq \left( \| \UB^{\frac{1}{2}} \|_2 \| \KB_j \|_2  \| \VB_j^{\frac{1}{2}}\|_2  \right)^2 =  \left( \frac{1}{\sqrt{cJ R_{\max}}}   R_j \sqrt{\frac{c}{R_j}} \right)^2 = \frac{1}{J}. \label{eq:primary}
\end{align}
For any $\xB \in \Rbb^{d}$, $\yB_j \in \Rbb^{m_j}$, we consider
\begin{equation}
-2\langle \xB, \KB_j \yB_j \rangle = -2 \langle \UB^{-\frac{1}{2}} \xB,\UB^{\frac{1}{2}} \KB_j \VB_j^{\frac{1}{2}} \VB_j^{-\frac{1}{2}} \yB_j \rangle.
\end{equation}
Applying the Cauchy-Schwarz inequality and the fact that $2ab \leq ha^2 + b^2/h$ for any $a,b$ and $h>0$, we obtain,
\begin{align}
-2\langle \xB, \KB_j \yB_j \rangle &\geq -2 \| \UB^{-\frac{1}{2}} \xB \|_2 \|  \UB^{\frac{1}{2}} \KB_j \VB_j^{\frac{1}{2}} \VB_j^{-\frac{1}{2}} \yB_j  \|_2 \\
& \geq - \left( \frac{1}{h} \langle \xB, \UB^{-1} \xB \rangle +   h \| \UB^{\frac{1}{2}} \KB_j \VB_j^{\frac{1}{2}} \|_2^2 \langle \yB_j, \VB_j^{-1} \yB_j \rangle   \right) \label{eq:cross1}
\end{align}
In view of the inequality~(\ref{eq:primary}), it is obvious that there exists certain $\epsilon > 0 $ such that the following equality holds,
\begin{equation}
(J+\epsilon)(1+\epsilon)\| \UB^{\frac{1}{2}} \KB_j \VB_j^{\frac{1}{2}} \|_2^2 = 1.
\end{equation}
Thanks to this equality, now we set $h = J+ \epsilon$, and the inequality~(\ref{eq:cross1}) can be further simplified,
\begin{align}
-2\langle \xB, \KB_j \yB_j \rangle &\geq - \left( \frac{1}{J+\epsilon} \langle \xB, \UB^{-1} \xB \rangle +   \frac{(J+\epsilon) \| \UB^{\frac{1}{2}} \KB_j \VB_j^{\frac{1}{2}} \|_2^2}{(J+\epsilon)(1+\epsilon)\| \UB^{\frac{1}{2}} \KB_j \VB_j^{\frac{1}{2}} \|_2^2} \langle \yB_j, \VB_j^{-1} \yB_j \rangle   \right) \\
&= - \left( \frac{1}{J+\epsilon} \langle \xB, \UB^{-1} \xB \rangle + \frac{1}{1+\epsilon} \langle \yB_j, \VB_j^{-1} \yB_j \rangle \right) \label{eq:cross2}
\end{align}
Let $\yB = \left(\yB_1, \dots, \yB_J\right) \in \Rbb^m$, and now we consider for any non-zero $(\xB, \yB) \in \Rbb^{d+m}$, the following inner product can be expanded,
\begin{align}
\langle (\xB, \yB), \PB (\xB, \yB) \rangle = \langle \xB, \UB^{-1} \xB \rangle + \sum_{j=1}^J \langle \yB_j, \VB_j^{-1} \yB_j \rangle - 2\sum_{j=1}^J \langle \xB, \KB_j \yB_j \rangle.
\end{align}
Inserting the inequality~(\ref{eq:cross2}) into the above equation, we obtain,
\begin{align}
\langle (\xB, \yB), \PB (\xB, \yB) \rangle &\geq  \langle \xB, \UB^{-1} \xB \rangle + \sum_{j=1}^J \langle \yB_j, \VB_j^{-1} \yB_j \rangle \\
&- \sum_{j=1}^J\left( \frac{1}{J+\epsilon} \langle \xB, \UB^{-1} \xB \rangle + \frac{1}{1+\epsilon} \langle \yB_j, \VB_j^{-1} \yB_j \rangle \right) \\
& = \frac{\epsilon}{J+\epsilon} \langle \xB, \UB^{-1} \xB \rangle + \frac{\epsilon}{1+\epsilon} \langle \yB_j, \VB_j^{-1} \yB_j \rangle > 0,
\end{align}
which guarantees the positive definiteness of the matrix $\PB$.
\end{proof}

Now we are ready to proof the Theorem \ref{th:main} in our paper:
%
%

\begin{proof}

Firstly, we analyze the value of the dual variable $\yB$ after $t$-th update in Algorithm 1. For any $i \in \{ 1,2,\dots, n \}$, let $\tilde{\yB}_i$ be the value of $\yB_i^{t+1}$ if $i \in S_t$, i.e.,
\begin{equation}
\tilde{\yB}_i = \argmin_{\yB_i} \phi_i^*(\yB_i) - \langle \overline{\xB}^t, \AB_i \yB_i \rangle +  \frac{1}{2\sigma_i} \| \yB_i - \yB_i^t\|_2^2
\end{equation}
Since $\phi^*(\cdot)$ is $\gamma$-strongly convex, thus the function to be minimized above is $(1/\sigma_i + \gamma)$-strongly convex. Then we have,
\begin{align}
\phi_i^*(\yB_i^{\star}) - \langle \overline{\xB}^t, \AB_i \yB_i^{\star} \rangle + \frac{1}{2\sigma_i} \| \yB_i^{\star} - \yB_i^t\|_2^2 & \geq \phi_i^*(\tilde{\yB}_i) - \langle \overline{\xB}^t, \AB_i \tilde{\yB}_i\rangle + \frac{1}{2\sigma_i} \| \tilde{\yB}_i - \yB_i^t\|_2^2 \nonumber \\
&+
\left( \frac{1}{\sigma_i} + \gamma \right) \frac{\| \tilde{\yB}_i - \yB_i^{\star}\|_2^2}{2}
\end{align}
Since the $(\xB^{\star}, \yB^{\star})$ is the saddle point, we can obtain following inequality,
\begin{equation}
\phi_i^*(\tilde{\yB}_i) - \langle \xB^{\star}, \AB_i \tilde{\yB}_i\rangle \geq  \phi_i^*(\yB_i^{\star}) - \langle \xB^{\star}, \AB_i \yB_i^{\star} \rangle
\end{equation}
Adding the two inequalities together, we have
\begin{equation}
\frac{\| \yB_i^t - \yB_i^{\star}\|_2^2}{2\sigma_i} \geq \left( \frac{1}{2\sigma_i} + \gamma \right)\| \tilde{\yB}_i - \yB_i^{\star}\|_2^2 +
\frac{1}{2\sigma_i} \| \tilde{\yB}_i - \yB_i^t\|_2^2 +  \langle \xB^{\star} - \overline{\xB}^t, \AB_i \left( \tilde{\yB}_i - \yB_i^{\star} \right) \rangle  \label{eq:dual_ineq1}
\end{equation}

In our algorithm, an index set $S_t$ is randomly chosen. For every specific index $i$, the event $i \in S_t$ happens with probability $m/n$. If $i \in S_t$, then $\yB_i^{t+1}$ is updated to the value $\tilde{\yB}_i^t$. Otherwise, $\yB_i^{t+1}$ is kept to be its old value $\yB_i^t$. Let $\xi_t$ be the random event that contains the set of all random variable before round $t$,
\begin{equation}
\xi_t = \{S_1, S_2, \dots, S_t  \},
\end{equation}
and then we have
\begin{align*}
\Ebb_{\xi_t} \left[ \| \yB_i^{t+1} - \yB_i^{\star} \|_2^2  \right] &= \frac{m}{n} \| \tilde{\yB}_i - \yB_i^{\star} \|_{2}^2 + \frac{n-m}{n} \| \yB_i^{t} - \yB_i^{\star} \|_{2}^2  \\
\Ebb_{\xi_t} \left[ \| \yB_i^{t+1} - \yB_i^{t} \|_2^2  \right] &= \frac{m}{n} \| \tilde{\yB}_i - \yB_i^{t} \|_{2}^2\\
\Ebb_{\xi_t} \left[ \yB_i^{t+1}  \right] &= \frac{m}{n}  \tilde{\yB}_i  + \frac{n-m}{n}  \yB_i^{t}
\end{align*}
Consequently, we can insert the representations of $\| \tilde{\yB}_i - \yB_i^{\star} \|_{2}^2$, $\| \tilde{\yB}_i - \yB_i^{t} \|_{2}^2$ and $\tilde{\yB}_i$ in terms of the above expectations into the inequality~(\ref{eq:dual_ineq1}),
\begin{align}
\left( \frac{n}{2m\sigma_i}  + \frac{n-m}{m} \gamma \right)\| \yB_i^t - \yB_i^{\star} \|_2^2 &\geq  \left( \frac{n}{2m\sigma_i}  + \frac{n}{m} \gamma \right) \Ebb_{\xi_t} \left[ \| \yB_i^{t+1} - \yB_i^{\star} \|_2^2 \right] \nonumber \\
&+ \frac{n}{2m \sigma_i} \Ebb_{\xi_t} \left[ \| \yB_i^{t+1} - \yB_i^{t} \|_2^2 \right] \nonumber \\
&+ \left< \xB^{\star} - \overline{\xB}^t, \AB_i \left( \yB_i - \yB_i^{\star} + \frac{n}{m} \Ebb_{\xi_t} \left[ \| \yB_i^{t+1} - \yB_i^{t} \|_2^2 \right]   \right) \right>
\end{align}
Then we add the above inequality from $i=1,2,\dots,n$, and divide both sides by $n$, and obtain
\begin{align}
\| \yB^t - \yB^{\star} \|_{\muB}^2 &\geq \Ebb_{\xi_t} \left[ \| \yB^{t+1} - \yB^{\star} \|_{\muB '}^2 \right] + \frac{1}{2 m }\Ebb_{\xi_t} \left[ \| \yB^{t+1} - \yB^{t} \|_{\sigmaB}^2 \right] \nonumber \\
&+ \Ebb_{\xi_t} \left[ \left< \xB^{\star} - \overline{\xB}^t,  \uB^t - \uB^{\star} + \frac{1}{m} \sum_{j\in S_t} \AB_j \left( \yB_j^{t+1} - \yB_j^{t} \right)  \right>  \right], \label{eq:dual_ineq2}
\end{align}
where $\mu_i = \frac{1}{2m \sigma_i} + \frac{n-m}{mn} \gamma$, $\mu_i ' = \frac{1}{2m \sigma_i} + \frac{\gamma}{m}$,  $\uB^{\star} = \frac{1}{n} \sum_{i=1}^n \AB_i \yB_i^{\star}$, and $\uB^{t} = \frac{1}{n} \sum_{i=1}^n \AB_i \yB_i^{t}$. In the crossing term between primal and dual variable, we use the fact that $\sum_{i=1}^n \AB_i (\yB_i^{t+1} - \yB_i^{t}) = \sum_{j\in S_t}  \AB_j (\yB_j^{t+1} - \yB_j^{t})$ since only the blocks in index set $S_t$ are chosen and updated in $t$-th update.

Now we characterize the $t$-th update of primal variable $\xB$. Following the same derivation for dual variable and using the assumption that $g(\cdot)$ is $\lambda$-strongly convex,  we can easily obtain
\begin{multline}
\frac{1}{2 \tau^t} \| \xB^t - \xB^{\star} \|_{2}^2 \geq \left( \frac{1}{2\tau^t} + \lambda \right) \| \xB^{t+1} - \xB^{\star} \|_{2}^2 + \frac{1}{2 \tau^t} \| \xB^{t+1} - \xB^{t} \|_{2}^2 \\
+  \left< \xB^{t+1} - \xB^t,  \uB^t - \uB^{\star} + \frac{1}{m} \sum_{j\in S_t} \AB_j \left( \yB_j^{t+1} - \yB_j^{t} \right)  \right>.  \label{eq:primal_ineq1}
\end{multline}
Taking expectation over both sides of the above inequality and adding it to the the inequality~(\ref{eq:dual_ineq2}), then we have
\begin{multline}
\frac{1}{2 \tau^t} \| \xB^t - \xB^{\star} \|_{2}^2 + \| \yB^t - \yB^{\star} \|_{\muB}^2  \geq \left( \frac{1}{2\tau^t} + \lambda \right) \Ebb_{\xi_t} \left[ \| \xB^{t+1} - \xB^{\star} \|_{2}^2 \right] + \Ebb_{\xi_t} \left[ \| \yB^{t+1} - \yB^{\star} \|_{\muB '}^2 \right]  \\
 +  \frac{1}{2 \tau^t} \Ebb_{\xi_t} \left[ \| \xB^{t+1} - \xB^{t} \|_{2}^2 \right] + \frac{1}{2 m }\Ebb_{\xi_t} \left[ \| \yB^{t+1} - \yB^{t} \|_{\sigmaB}^2 \right] \\
 + \Ebb_{\xi_t} \left[ \left< \xB^{t+1} - \xB^t - \theta^t \left( \xB^t - \xB^{t-1} \right), \AB \left( \frac{1}{n}(\yB^t - \yB^{\star}) + \frac{1}{m}(\yB^{t+1} - \yB^t)  \right)   \right>  \right], \label{eq:primal_dual_ineq}
\end{multline}
where the matrix $\AB = \left[ \AB_1, \AB_2, \dots, \AB_n \right] \in \Rbb^{d \times n}$.

Now we focus on the most crucial part of the proof: bounding the last term of R.H.S. of the above inequality~(\ref{eq:primal_dual_ineq}). Firstly we rearrange this crossing term as follows,
\begin{align}
 &\left< \xB^{t+1} - \xB^t - \theta^t \left( \xB^t - \xB^{t-1} \right), \AB \left( \frac{1}{n}(\yB^t - \yB^{\star}) + \frac{1}{m}(\yB^{t+1} - \yB^t)  \right)   \right> \nonumber \\
 & = \frac{1}{n} \left< \xB^{t+1} - \xB^t, \AB \left( \yB^{t+1} - \yB^t \right)   \right>
 - \frac{\theta^t}{n} \left< \xB^{t} - \xB^{t-1}, \AB \left( \yB^{t} - \yB^{\star} \right) \right>  \nonumber \\
 &+ \frac{n-m}{mn} \left< \xB^{t+1} - \xB^t, \AB \left( \yB^{t+1} - \yB^t \right)  \right> - \frac{\theta^t}{m} \left< \xB^t - \xB^{t-1}, \AB \left( \yB^{t+1} - \yB^t \right)   \right>. \label{eq:crossing_rearrange}
\end{align}
Given the parameter configuration in Eq~(\ref{eq:tau}) and (\ref{eq:sigma}), we consider the following symmetric matrix,
\begin{equation}
\PB =
\begin{bmatrix}
\frac{m}{2\tau^t} \IB & -\AB_{S_t} \\
-\AB_{S_t}^T &  \frac{1}{2\diag (\sigmaB_{S_t})}
\end{bmatrix}
\end{equation}
Applying the Lemma 1, we can guarantee the positive definiteness of the matrix $\PB$, which naturally leads the following inequality,
\begin{equation}
\frac{m}{4 \tau^t} \| \xB^{t+1} - \xB^t \|_2^2 + \sum_{i\in S_t} \frac{1}{4\sigma_i} \| \yB_i^{t+1} - \yB_i^{t}\|_2^2 \geq \left< \xB^{t+1} - \xB^t, \sum_{i\in S_t} \AB_i \left( \yB_i^{t+1} - \yB_i^t \right)  \right>
\end{equation}
Similarly, we can also obtain
\begin{equation}
\frac{m}{4 \tau^t} \| \xB^{t+1} - \xB^t \|_2^2 + \sum_{i\in S_t} \frac{1}{4\sigma_i} \| \yB_i^{t+1} - \yB_i^{t}\|_2^2 \geq -\left< \xB^{t+1} - \xB^t, \sum_{i\in S_t} \AB_i \left( \yB_i^{t+1} - \yB_i^t \right)  \right>
\end{equation}
Taking the expectation for both sides of the above two equalities and using the facts that
\begin{align*}
&\Ebb_{\xi_t} \left[ \frac{m}{4 \tau^t} \| \xB^{t+1} - \xB^t \|_2^2 + \| \yB^{t+1} - \yB^t \|_{\frac{1}{4 \diag (\sigmaB)}}^2 \right] \\  &=\Ebb_{\xi_t} \left[ \frac{m}{4 \tau^t} \| \xB^{t+1} - \xB^t \|_2^2 + \sum_{i\in S_t} \frac{1}{4\sigma_i} \| \yB_i^{t+1} - \yB_i^{t}\|_2^2 \right], \\
&\Ebb_{\xi_t} \left[ \left| \left< \xB^{t+1} - \xB^t, \AB \left( \yB^{t+1} - \yB^t \right)  \right> \right| \right] \\
&= \Ebb_{\xi_t} \left[ \left| \left< \xB^{t+1} - \xB^t, \sum_{i\in S_t} \AB_i \left( \yB_i^{t+1} - \yB_i^t \right)  \right> \right| \right],
\end{align*}
we have that
\begin{equation}
\Ebb_{\xi_t} \left[ \left| \left< \xB^{t+1} - \xB^t, \AB \left( \yB^{t+1} - \yB^t \right)  \right> \right| \right] \leq \Ebb_{\xi_t} \left[ \frac{m}{4 \tau^t} \| \xB^{t+1} - \xB^t \|_2^2 + \| \yB^{t+1} - \yB^t \|_{\frac{1}{4 \diag (\sigmaB)}}^2 \right]
\end{equation}
Similarly, we can obtain                                                                                                                                                                                                                                                                                                                                                                                                                                                                                                                                                          
\begin{equation}
\Ebb_{\xi_t} \left[ \left| \left< \xB^{t} - \xB^{t-1}, \AB \left( \yB^{t+1} - \yB^t \right)  \right> \right| \right] \leq \Ebb_{\xi_t} \left[ \frac{m}{4 \tau^t} \| \xB^{t} - \xB^{t-1} \|_2^2 + \| \yB^{t+1} - \yB^t \|_{\frac{1}{4 \diag (\sigmaB)}}^2 \right]
\end{equation}
Therefore,
\begin{align}
\Ebb_{\xi_t} \left[ \left< \xB^{t+1} - \xB^t, \AB \left( \yB^{t+1} - \yB^t \right)  \right>  \right]  \geq &- \Ebb_{\xi_t} \left[ \frac{m}{4 \tau^t} \| \xB^{t+1} - \xB^t \|_2^2 \right] \nonumber \\
&- \Ebb_{\xi_t} \left[  \| \yB^{t+1} - \yB^t \|_{\frac{1}{4 \diag (\sigmaB)}}^2 \right] \label{eq:bound1}\\
\Ebb_{\xi_t} \left[ \left< \xB^{t} - \xB^{t-1}, \AB \left( \yB^{t+1} - \yB^t \right)  \right>  \right]  \geq &- \Ebb_{\xi_t} \left[ \frac{m}{4 \tau^t} \| \xB^{t} - \xB^{t-1} \|_2^2 \right] \nonumber \\
&- \Ebb_{\xi_t} \left[  \| \yB^{t+1} - \yB^t \|_{\frac{1}{4 \diag (\sigmaB)}}^2 \right] \label{eq:bound2}
\end{align}
Now we insert the Eq.~(\ref{eq:crossing_rearrange}) into the inequality~(\ref{eq:primal_dual_ineq}), and then apply the two bounds (\ref{eq:bound1}) and (\ref{eq:bound2}), we have
\begin{multline}
\frac{1}{2 \tau^t} \| \xB^t - \xB^{\star} \|_{2}^2 + \| \yB^t - \yB^{\star} \|_{\muB}^2  \geq \left( \frac{1}{2\tau^t} + \lambda \right) \Ebb_{\xi_t} \left[ \| \xB^{t+1} - \xB^{\star} \|_{2}^2 \right] + \Ebb_{\xi_t} \left[ \| \yB^{t+1} - \yB^{\star} \|_{\muB '}^2 \right] \nonumber \\
 +  \frac{1}{4 \tau^t} \Ebb_{\xi_t} \left[ \| \xB^{t+1} - \xB^{t} \|_{2}^2 \right] + \frac{1}{n}\Ebb_{\xi_t} \left[ \left< \xB^{t+1} - \xB^t, \AB \left(\xB^{t+1} - \xB^t \right) \right> \right] \nonumber\\
 - \frac{\theta^t}{4 \tau^t} \| \xB^{t} - \xB^{t-1} \|_2^2  - \frac{\theta^t}{n} \left< \xB^t - \xB^{t-1}, \AB \left(\yB^t - \yB^{\star} \right) \right> \nonumber
  + \frac{1- \theta^t + m/n}{4 m }\Ebb_{\xi_t} \left[ \| \yB^{t+1} - \yB^{t} \|_{\sigmaB}^2 \right]. 
\end{multline}
Recall the configuration for $\theta^t$ in Eq.~(\ref{eq:theta}), the last term of R.H.S. of the above inequality is non-negative, and  can be bounded away. Then we have the following,
\begin{multline}
\frac{1}{2 \tau^t} \| \xB^t - \xB^{\star} \|_{2}^2 + \| \yB^t - \yB^{\star} \|_{\muB}^2 + \frac{\theta^t}{4 \tau^t} \| \xB^{t} + \xB^{t-1} \|_2^2  + \frac{\theta^t}{n} \left< \xB^t - \xB^{t-1}, \AB \left(\yB^t - \yB^{\star} \right) \right>  \geq \\
 \left( \frac{1}{2\tau^t} + \lambda \right) \Ebb_{\xi_t} \left[ \| \xB^{t+1} - \xB^{\star} \|_{2}^2 \right]
  + \Ebb_{\xi_t} \left[ \| \yB^{t+1} - \yB^{\star} \|_{\muB '}^2 \right]  +  \frac{1}{4 \tau^t} \Ebb_{\xi_t} \left[ \| \xB^{t+1} - \xB^{t} \|_{2}^2 \right] \\
  + \frac{1}{n} \Ebb_{\xi_t} \left[ \left< \xB^{t+1} - \xB^t, \AB \left(\xB^{t+1} - \xB^t \right) \right> \right]
  \triangleq \Delta^{t+1} \label{eq:primal_dual_ineq2}
\end{multline}
According to the defined sequence $\Delta^{t+1}$, we have
\begin{align}
\theta^t \Delta^t = &\theta^t\left( \frac{1}{2\tau^t} + \lambda \right) \| \xB^t - \xB^{\star} \|_{2}^2 + \theta^t \| \yB^t - \yB^{\star} \|_{\muB '}^2 \nonumber \\
&+ \frac{\theta^t}{4 \tau^t} \| \xB^{t} + \xB^{t-1} \|_2^2  + \frac{\theta^t}{n} \left< \xB^t - \xB^{t-1}, \AB \left(\yB^t - \yB^{\star} \right) \right> \label{eq:product_seq}
\end{align}
According to the parameter configuration for $\tau^t$, $\sigma_i$ and $\theta^t$, we can easily verify that
\begin{align*}
\theta^t\left( \frac{1}{2\tau^t} + \lambda \right) &\geq \ \frac{1}{2 \tau^t} \\
\theta^t \mu_i ' &\geq \mu_i
\end{align*}
Combining these two inequalities with the inequality~(\ref{eq:primal_dual_ineq2}) and Eq.~(\ref{eq:product_seq}), we have
\begin{equation}
\Delta^{t+1} \leq \theta^t \Delta^{t}.
\end{equation}
Consider $t=0, 1, \dots, T$, the above inequality implies
\begin{align}
\left( \frac{1}{2\tau^T} + \lambda \right) \Ebb \left[ \| \xB^{T} - \xB^{\star} \|_{2}^2 \right] + \Ebb \left[ \| \yB^{T} - \yB^{\star} \|_{\muB '}^2 \right]  +  \frac{1}{4 \tau^T} \Ebb \left[ \| \xB^{T} - \xB^{T-1} \|_{2}^2 \right] \nonumber \\
+ \frac{1}{n} \Ebb \left[ \left< \xB^{T} - \xB^{T-1}, \AB \left(\yB^{T} - \yB^{\star} \right) \right> \right] \leq \left( \prod_{t=1}^T \theta^t \right) \Delta^{0}, \label{eq:final}
\end{align}
where
\begin{equation}
\Delta^0 = \left( \frac{1}{2\tau^0} + \lambda \right) \Ebb \left[ \| \xB^{0} - \xB^{\star} \|_{2}^2 \right] + \Ebb \left[ \| \yB^{0} - \yB^{\star} \|_{\muB '}^2 \right].
\end{equation}
Consider the following matrix
\begin{equation}
\QB =
\begin{bmatrix}
\frac{n}{2\tau^T} \IB & \pm \AB_{S_T} \\
\pm \AB_{S_T}^T &  \frac{n}{2m\diag (\sigmaB_{S_t})}
\end{bmatrix}
\end{equation}
Applying the Lemma 1 again, we can guarantee the positive definiteness of the matrix, which implies that
\begin{equation}
\pm \frac{1}{n}\left< \xB^{T} - \xB^{T-1}, \sum_{i \in S_T} \AB_i \left(\yB_i^{T} - \yB_i^{\star} \right) \right>  \leq \frac{1}{4 \tau^T} \| \xB^T - \xB^{T-1}\|_2^2 + \frac{1}{4m} \sum_{i \in S_T} \frac{1}{\sigma_i} \| \yB_i^T - \yB_i^{\star} \|_2^2
\end{equation}
Taking expectation,
\begin{align}
\frac{1}{n} \Ebb \left[ \left|  \left< \xB^{T} - \xB^{T-1}, \AB \left(\yB^{T} - \yB^{\star} \right) \right> \right| \right] \leq &  \frac{1}{4 \tau^T} \Ebb \left[ \| \xB^{T} - \xB^{T-1} \|_{2}^2 \right] \nonumber \\
&+ \frac{1}{4m} \Ebb \left[ \| \yB^T -\yB^{\star} \|_{1/ \diag \left( \sigmaB \right)}\right]
\end{align}
Thus,
\begin{align}
\frac{1}{n} \Ebb \left[ \left< \xB^{T} - \xB^{T-1}, \AB \left(\yB^{T} - \yB^{\star} \right) \right> \right] \geq &  - \frac{1}{4 \tau^T} \Ebb \left[ \| \xB^{T} - \xB^{T-1} \|_{2}^2 \right] \nonumber \\
&- \frac{1}{4m} \Ebb \left[ \| \yB^T -\yB^{\star} \|_{1/ \diag \left( \sigmaB \right)}\right]
\end{align}
Then combining the above inequality with inequality (\ref{eq:final}), we have
\begin{align}
&\left( \frac{1}{2\tau^T} + \lambda  \right) \Ebb \left[ \| \xB^T - \xB^{\star} \|_2^2 \right] + \Ebb \left[ \| \yB^T - \yB^{\star} \|_{\nuB}^2 \right] \nonumber \\
\leq & \left( \prod_{t=1}^T  \theta^t\right) \left( \left( \frac{1}{2\tau^T} + \lambda  \right) \| \xB^0 - \xB^{\star} \|_2^2   +  \| \yB^0 - \yB^{\star} \|_{\nuB '}^2 \right),
\end{align}
where $\nu_i = \frac{1/(4\sigma_i) + \gamma}{m} $, $\nu_i ' = \mu_i ' = \frac{1/(2\sigma_i) + \gamma}{m}$, and $\| \yB^T - \yB^{\star} \|_{\nuB}^2 = \sum_{i=1}^n \nu_i \|\yB^T_i - \yB^{\star}_i \|_2^2 $, which completes the proof.
\end{proof}

\end{document}